\documentclass{article}

\usepackage{arxiv}

\usepackage[utf8]{inputenc} 
\usepackage[T1]{fontenc}    
\usepackage[hidelinks]{hyperref}       
\usepackage{url}            

\usepackage{amsmath,amssymb,amsfonts,amsthm}
\usepackage{natbib}

\usepackage{authblk}

\usepackage{float}
\usepackage{algorithmic}
\usepackage{graphicx}
\usepackage{textcomp}
\usepackage{xcolor}
\usepackage{multirow}
\usepackage{booktabs}
\usepackage[algo2e,ruled,linesnumbered]{algorithm2e}
\usepackage{algorithm}
\usepackage[hidelinks]{hyperref}
\usepackage{cleveref}
\usepackage{comment}

\usepackage{url}
\usepackage{graphicx}
\usepackage{amsmath}
\usepackage{amsthm}
\usepackage{amsfonts}
\usepackage{caption}
\usepackage{multirow}
\usepackage{threeparttable}
\usepackage{enumitem}
\usepackage{footmisc}
\usepackage{subcaption}

\def\BibTeX{{\rm B\kern-.05em{\sc i\kern-.025em b}\kern-.08em
    T\kern-.1667em\lower.7ex\hbox{E}\kern-.125emX}}

\newtheorem{hyp}{Hypothesis}
\newtheorem{thm}{Theorem}
\newtheorem{prop}[thm]{Proposition}
\newtheorem{lemma}[thm]{Lemma}
\newtheorem{cor}[thm]{Corollary}
\newtheorem{definition}{Definition}
\newcommand{\proposal}{LAM}
\newcommand{\proposalextended}{Local Attention Mechanism}

\title{Local Attention Mechanism: Boosting the Transformer Architecture for Long-Sequence Time Series Forecasting}

\author{
 Ignacio Aguilera-Martos \textsuperscript{*}, Andrés Herrera-Poyatos \textsuperscript{*}, Julián Luengo \textsuperscript{*},   Francisco Herrera \textsuperscript{*} \textsuperscript{$\dagger$}
}

\affil{\textsuperscript{*} Andalusian Institute of Data Science and Computational Intelligence (DaSCI), University of Granada, Spain. \\ Emails: \texttt{divadhp@ugr.es}, \texttt{jdrull@correo.ugr.es}, \texttt{rosana@ugr.es}, \texttt{herrera@decsai.ugr.es}, \texttt{andreshp@ugr.es}}

\affil{\textsuperscript{$\dagger$} ADIA Lab, AI Maryah Island, Abu Dhabi, United Arab Emirates}

\renewcommand{\shorttitle}{Local Attention Mechanism for Long-Sequence Time Series Forecasting}

\hypersetup{
pdftitle={A template for the arxiv style},
pdfsubject={q-bio.NC, q-bio.QM},
pdfauthor={David S.~Hippocampus, Elias D.~Striatum},
pdfkeywords={First keyword, Second keyword, More},
}

\begin{document}

\maketitle

\begin{abstract}
Transformers have become the leading choice in natural language processing over other deep learning architectures. This trend has also permeated the field of time series analysis, especially for long-horizon forecasting, showcasing promising results both in performance and running time. 

In this paper, we introduce \proposalextended\ (\proposal), an efficient attention mechanism tailored for time series analysis. This mechanism exploits the continuity properties of time series to reduce the number of attention scores computed. We present an  algorithm for implementing \proposal\ in tensor algebra that runs in time and memory $\Theta(n \log n)$, significantly improving upon the $\Theta(n^2)$ time and memory complexity of traditional attention mechanisms. We also note the lack of proper datasets to evaluate long-horizon forecast models. Thus, we propose a novel set of datasets to improve the evaluation of models addressing long-horizon forecasting challenges.

Our experimental analysis demonstrates that the vanilla transformer architecture magnified with \proposal\ surpasses state-of-the-art models, including the vanilla attention mechanism. These results confirm the effectiveness of our approach and highlight a range of future challenges in long-sequence time series forecasting.
\end{abstract}


\section{Introduction}
\label{sec:introduction}

We are currently living in the information age, with vast amounts of data being continuously generated and requiring processing to maximise their utility. From this type of data, time series emerge as the natural choice for representation in several applications~\citep{lin2003symbolic}. Time series, which consist of sequentially related data, are commonly encountered in sectors such as finance~\citep{zhang2024deep}, industry~\citep{yan2024comprehensive}, telecommunications~\citep{di2024hybrid}, electricity~\citep{gulay2024forecasting}, and climatology~\citep{karevan2020transductive}. 

Time series forecasting involves predicting future time instances based on past information. An specific scenario of this problem is Long Sequence Time Series Forecasting (LSTF)~\citep{chen2023long}, which implies a longer prediction horizon and represents the most interesting and complex problem in time series forecasting. One of the main concerns of the LSTF problem is capturing long-term trends in the data. To adequately predict these long-term patterns, a substantial number of instances is required to model complex behaviour effectively.

Traditionally, the forecasting problem has been addressed using statistical methods like ARIMA~\citep{arima}. However, these models are prone to significant long-term error accumulation~\citep{arima-limitations}, prompting the exploration of new alternatives within the field of Deep Learning. Convolutional models for time series~\citep{8489399,electronics8080876} leverage the concept correlation in consecutive time steps. Various models have been developed based on this architecture, typically incorporating recurrent layers~\citep{HEWAMALAGE2021388} such as Long-Short Term Memory (LSTM)~\citep{10.1115/1.4056122,DUDUKCU2023109945,AGUILERAMARTOS2023223}. This type of network faces clear limitations due to the complexity of recurrent operations~\citep{lipton2015critical} and the lack of capacity to capture long-term time dependencies~\citep{al2023lstm}. 

Transformers, originally proposed for natural language processing~\citep{vaswani2017attention}, have been incorporated into the field of time series and forecasting to address these needs. Nonetheless, current transformer models present two main limitations. First of all, the computational evaluation of the attention mechanism requires significant time and memory due to its quadratic complexity on the length of the input. Secondly, the utilisation of attention mechanisms results in a significant loss of positional information, which poses a critical challenge for time series forecasting~\citep{zeng2023transformers}. As a first step to address the computational cost issue, an efficient attention mechanism based on probabilistic sampling of attentions has been proposed in~\citep{zhou2023informer}. However, to date, methodologies employing transformers for time-series forecasting have not sufficiently mitigated the loss of positional information~\citep{zeng2023transformers}. Consequently, emerges a necessity for an innovative attention mechanism that leverages the local and long-term temporal dependency properties of time series, does not incur information loss and minimises computational and memory overheads.

Previously existing transformer models within the field of time series do not exploit this conception of locality and present experimental studies with data sets of limited size, not only in terms of features but particularly in terms of instances. If the goal is capturing long-term dependencies, large datasets are needed to adequately train the model weights. The primary defining feature of a time series is the dependence of past time steps on the current time instant~\citep{bloomfield2004fourier}. This dependence is more likely to be higher in closer time steps, specially in the case where the time series comes from discretizising a continuous stochastic process.

This paper aims to achieve two primary objectives: 
\begin{itemize}
    \item We introduce \proposalextended\  (\proposal), a locality-based attention mechanism designed to enhance both the efficiency and performance of time series transformers. This novel attention mechanism reduces computational and memory complexity to $\Theta(n \log n)$ by excluding non-relevant computations in the attention matrix through the use of an additive mask. The key contribution here is a derivation of this attention mechanism that can be written in tensor algebra in such a way that only $\Theta(n \log n)$ memory is needed.
    \item We propose a more suitable benchmark for the LSTF problem, comprising four interdisciplinary and substantially larger datasets than previous studies. This new experimental environments encompass datasets related to taxi usage, electricity consumption and industrial real-world problems.
\end{itemize}

To validate the proposed approach, we have divided the experimental framework in three scenarios. First, \proposal\ will be compared within the state-of-the-art in transformer models and most used datasets. Secondly, Informer's attention mechanism (ProbAttention) and \proposal\ will be compared under the same architecture to provide a fair comparison. Finally, we propose a new dataset collection for the LSTF problem, requiring a wider prediction horizon and larger datasets. Thanks to this comprehensive comparison, we cover all possible cases and provide robust results.

The rest of this paper is structured as follows. In Section~\ref{sec:time-series-forecasting}, we discuss the background on transformers and the LSTF problem. In Section~\ref{sec:local-attention}, we detail \proposal\, introducing all relevant details. In Section~\ref{sec:informer-experimental-setup}, we compare our approach with the state-of-the-art methods using currently employed LSTF datasets. In Section~\ref{sec:informer-experimental-setup:ProbAttention}, we compare ProbAttention and \proposal\ within the vanilla transformer architecture. In Section~\ref{sec:lstf-experimental-setup}, we extend the datasets used to better represent the complexities of the LSTF problem and establish a more comprehensive benchmark. Finally, we present our conclusions and outline future work in Section~\ref{sec:conclusions-future-work}.

\section{Transformers for Long Sequence Time-series Forecasting}
\label{sec:time-series-forecasting}
\label{sec:time-series-forecasting:subsec:lstf}

A time series is a sequence of random vectors $X_t \in \mathbb{R}^d$ recorded at regular time intervals, typically denoted as $\{X_t\}_{t=1}^T$. We use $X_t$ to refer to the random vector and $x_t$ to represent the observation at time $t$~\citep{lai2018modeling}. Time series data is characterised by its temporal dependency, where each observation relies on previous values, potentially exhibiting long-term patterns~\citep{hyndman2018forecasting}. Such data often features distinct elements: trend (long-term data direction), seasonality (predictable periodic fluctuations) and cyclic patterns (less regular variations influenced by external factors). In this work we focus on time series forecasting, which corresponds to the problem of estimating $\mathbb{E} [ X_{t+1} \vert X_{1}=x_1, \ldots, X_{t}=x_t ]$. Specifically, we address LSTF problem~\citep{chen2023long}, where the goal is predicting long sequences of future data points in datasets with extensive historical records, requiring of the application of advanced techniques to understand complex patterns and trends.

In the rest of this section, we describe the elements of the vanilla transformer, which will serve as the base model for our experimentation, and the probabilistic attention mechanism of~\citep{zhou2023informer}. This section is organised as follows. In Section~\ref{sec:time-series-forecasting:subsec:forecasting} we present the background of time series forecasting using neural networks. In Section~\ref{sec:trans:full-attention} we briefly recapitulate the canonical attention mechanism introduced in~\citep{vaswani2017attention}. In Section~\ref{sec:trans:prob-attention} we summarise the concept of probabilistic attention and point out some of its weaknesses. Finally, in Section~\ref{sec:trans:encoder-decoder} we summarise the encoder-architecture of the vanilla transformer along with the positional encoding.

\subsection{Time series forecasting using neural networks}
\label{sec:time-series-forecasting:subsec:forecasting}

Neural networks, particularly one-dimensional convolutional neural networks (1D CNNs) and Long Short-Term Memory (LSTM) networks, offer black-box solutions for time series forecasting~\citep{KhanZulfiqar,KimTaeYoung}. 1D CNNs use convolutional layers to identify local patterns and hierarchies within time series data, but are unable to capture longer time dependencies \citep{al2023lstm}. Indeed, short-term prediction are greatly affected by locality, that is, the last few elements of the sequence greatly determine the trend of the model and, thus, the short-term prediction, unless seasonality and cyclic patterns play a role at that specific timestamp. LSTMs are able to take into account seasonality and cyclic patterns for medium-term dependencies thanks to the use memory cells, overcoming the vanishing gradient problem via gating mechanisms. However, the recursive nature of LSTMs and the memory cells limitations makes them less efficient and effective for long-term predictions~\citep{zhou2023informer}. 

Transformers have surpassed RNNs and LSTMs in several applications, such as NLP, thanks to their ability to model long-range dependencies, learn from large datasets, and parallelise computations efficiently. Motivated by these properties, transformers have been deemed as excellent candidates to the LSTF problem~\citep{li2019enhancing, zhou2023informer}. Indeed, unlike traditional methods that rely on fixed time windows, lagged observations and/or recursive neural networks, transformers can process long sequences efficiently, making them suitable for both short-term and long-term forecasting tasks \citep{lim2021temporal}.

Transformers are complex encoder-decoder architectures that encode the sequence of samples and use the extracted information to predict the several future elements of the sequence via a decoder. The key element to understand these architectures is the attention mechanism, see~\citep{vaswani2017attention}. In a nutshell, the main conclusion of~\citep{vaswani2017attention} is that performance in NLP applications can be significantly improved by eliminating the recurrence structure of deep learning models and instead focusing exclusively on exploiting the attention mechanism to extract relevant information from sequences of data. In this work we refer to this encoding-decoding model as the \emph{vanilla transformer}. 

When it comes to the application of transformers to time-series problems, progress has been significantly slower than in NLP. This is due to the fact that the attention mechanism presents two main limitations: loss of positional information, and computational and memory cost~\citep{li2019enhancing}. The former problem is solved in the application of the vanilla transformer to NLP problems via a positional encoding of the input, and residual connections. However, the validity of this approach in the case of time series remains controversial, see, for instance~\citep{zeng2023transformers}, suggesting that further research is needed. On the other hand, improving the computational and memory cost of the attention mechanism has actually been addressed in the time-series literature. The evaluation of the original attention mechanism has time and memory complexity $\Theta(n^2)$, where $n$ is the width of this layer. As a consequence of the memory bottleneck, not every data scientist has access to the computing resources needed to train and evaluate a transformer model with the full attention mechanism on large sequences of data, preventing transformers from becoming a machine learning state of the art in time series forecasting in practice. The most remarkable solution to accelerate the attention mechanism for time series forecasting is that of~\citep{zhou2023informer}, where the authors propose a probabilistic attention mechanism that runs in time and memory $\Theta(n \log n)$. Other proposed solutions to this problem involve applying other sorts of mechanisms to the sequence of data instead of the attention mechanism, see for instance~\citep{zhou2022fedformer}.

\subsection{Attention Mechanism}
\label{sec:time-series-forecasting:subsec:transformers-for-time-series-forecasting:subsubsec:attention}
\label{sec:trans:full-attention}

The attention mechanism is a layer of a neural network whose inputs consist of three sequences with $n$ elements: the lists of queries and keys, each entry being a vector of dimension $d_q$, and the list of values, each entry having dimension $d_v$. For example, when applied to time series as a first layer, these three sequences could be set to the input of model, that is, $n$ consecutive data points of the time series, $X = (x_1, x_2, \ldots, x_n)$. In this particular case, $d_q = d_v = d_{\operatorname{model}}$, so each of the three inputs of the attention layer is actually an $n \times d_{\operatorname{model}}$ matrix. In the general case, the queries, keys and values are matrices $Q$, $K$ and $V$ with dimensions $n \times d_q, n \times d_q$ and $n \times d_v$, respectively.

The \emph{attention matrix} of $Q$ and $K$ is defined as
\begin{equation}
    \operatorname{FullAttentionScores(Q, K)} = \operatorname{softmax}\left( \frac{QK^T}{\sqrt{d_q}}\right),
\end{equation}
which has dimension $n \times n$.  Note that the $i$-th row of $QK^T$ consists of the scalar products of the $i$-th query with all the keys. Provided that the scalar product of two vectors $a$ and $b$ satisfies $\langle a, b \rangle = \lVert a \rVert_2 \lVert b \rVert_2 \cos \theta$, where $\theta$ is the angle between the vectors $a$ and $b$. Thus, the $i$-th row of $Q K^T$ ranges from $-\lVert Q_i \rVert_2 \lVert K_j \rVert_2$ to $\lVert Q_i \rVert_2 \lVert K_j \rVert_2$. The softmax function here guarantees that attention matrix is stochastic, that is, each row consists of non-negative real numbers that add up to one. The entries of the attention matrix are called \emph{attention scores}. The canonical attention mechanism (refereed in this work as full attention mechanism) is then defined as 
\begin{equation}
    \operatorname{FullAttn(Q, K, V)} = \operatorname{softmax}\left( \frac{QK^T}{\sqrt{d_q}}\right) V,
\end{equation}
so the $i$-th output is a weighted average of the values. The scale factor $1 / 
\sqrt{d_q}$ does not change the output of $\operatorname{softmax}$; its purpose is to help scaling the gradients of $\operatorname{softmax}$ during training in order to avoid gradients too close to $0$, which may decrease the rate at which a transformer is able to learn from data.

The most common application of attention in transformers corresponds to the case $Q = K$. In this setting, the full attention mechanism resembles a smooth non-parametric regression~\citep{watson1964smooth, johnston1979smooth} of $V$ in terms of $Q = K$, ~\citep{nguyen2022fourierformer}, effectively providing a smoothing of $V$. We refer to \emph{self-attention} as the case when $Q = K = V$. Intuitively, self-attention weighs in the similarity of different elements in a sequence and smooths the input data. Empirically this process has been shown to be highly effective in  capturing both local and global patterns in data~\citep{feng2024attention}. However, this intuition has not been corroborated from a theoretical point of view, where mathematical results have only managed to unveil the Lipschitz constant of the attention mechanism~\citep{vuckovic2020mathematical}.

Transformers use an improvement of the attention mechanism, known as multi-head attention. Given an attention mechanism $\operatorname{Attn(\cdot, \cdot, \cdot)}$, which may be $\operatorname{FullAttn}$ or one of the other mechanisms that we will introduce later in this work. The idea here is applying linear projections (on the dimension of the characteristics) to learn $h$ representations of the inputs $Q$, $K$ and $V$. The weights of these linear projections will be optimised in the learning phase, aiming to learn representations that boost the performance of the attention mechanism. Then the attention mechanism is applied to each of the $h$ triples formed by the projections of $Q$, $K$ and $V$. Here $h$ is called the \emph{number of heads}. Formally, for positive integers $d_a$ and $h$, the \emph{multi-head attention layer} is defined as
\begin{equation*}
\begin{aligned}
    \operatorname{MultiHead}(Q, K, V) & = \operatorname{Concat}(head_1, \ldots, head_h) W^O, \\
    \text{where } head_i & = \operatorname{Attn}(Q W_i^Q, K W_i^K, V W_i^V),
\end{aligned}
\end{equation*}
and the matrices $W_i^Q$, $W_i^K$, $W_i^V$ and $W_i^O$ are parameter matrices of the model with $W_i^Q \in \mathbb{R}^{d_a \times d_q}$, $W_i^K\in \mathbb{R}^{d_a \times d_q}$, $W_i^V\in \mathbb{R}^{d_a \times d_v}$ and $W_i^O\in \mathbb{R}^{h d_a \times d_v}$. Here the $\operatorname{Concat}$ operator concatenates tensors $head_1, head_2, \ldots, head_h$ on the last dimension.

One usually picks $h$ and $d_a$ such that $h d_a = d_v$. In the time series literature  $h$ is typically set to the number of variables in the model, so $d_a = 1$. We highlight that the number of parameters of this layer is $d_a ( 2 d_q + (h+1) d_v)$, which does not depend on $n$. Thus, the multi-head attention mechanism allows us to add extra layers to the model without significantly increasing the number of parameters. 

A critique to the full attention mechanism is the potential loss of positional information. First of all, note that the output of $\operatorname{FullAttn(Q, K, V)}$ is a matrix whose $i$-th row is $\sum_{j = 1} a_{i,j} V_j$. This average $\sum_{j = 1} a_{i,j} V_j$ results in a partial loss of position as it may be the case that $a_{i,j}$ is large for a $V_j$ that is far away from $V_i$ in the input sequence. That is, we lose the sense of temporal locality of the values as we average values with entries far away in the input sequence. We formalise this loss of local information is in Proposition~\ref{prop:perm}, whose proof can be found in Appendix~\ref{appendix:theoretical-proofs}. First, let us introduce some notation. For a permutation $\pi \colon \{0, 1, \ldots, n-1\} \to \{0, 1, \ldots, n-1\}$, the \emph{permutation matrix} $P_\pi$ is the matrix with dimensions $n \times n$ defined as $P_\pi[i,j] = 1$ for $j = \pi(i)$ and $P_\pi[i,j] = 0$ otherwise. 

\begin{prop} \label{prop:perm}
In the case of the full attention mechanism, for any positive integers $h$ and $d_a$ and any permutation matrix $P_\pi$, we have 
\begin{equation*}
    \operatorname{MultiHead}(P_\pi Q, P_\pi K, P_\pi V) =  P_\pi \operatorname{MultiHead}(Q, K, V).
\end{equation*}
\end{prop}

As we will see in Section~\ref{sec:trans:encoder-decoder}, usually only the last layer of the vanilla transformer is a fully-linear layer; the rest of layers are attention layers with projections acting on the dimension of the characteristics. Therefore, positional information is only exploited in these last fully-linear layers. However, during the repeated application of attention layers positional information is reduced due to the averages carried out in the full attention mechanism. This issue is addressed by the implementation of the positional encoding, which aims to minimise this reduction of positional information when applying attention layers. However, it seems that, in the case of time series, positional encoding is not enough. Indeed, in experiments, on average a vanilla transformer trained for a particular time series produces similar predictions on an original input and a permuted version of this input ~\citep{zeng2023transformers}.

\subsection{Probabilistic Attention}
\label{sec:time-series-forecasting:subsec:probabilistic-attention}
\label{sec:trans:prob-attention}

Among state-of-the-art models, Informer is distinguished as one of the most frequently cited and top-performing~\citep{zhou2023informer}. The authors propose a probabilistic attention mechanism that runs in time and memory $\Theta(n \log n)$. The key idea here is selecting $u$ queries, where $u = 5 \lceil \log n \rceil$, and setting the rest of the queries to zeroes. To select these queries, the authors introduce a sparsity measurement for the queries, which approximates the Kullback-Leibler divergence of the attention scores of a query and the uniform distribution. Finding the queries whose attention scores have largest sparsity measurement would require to compute the whole attention matrix to begin with. Instead, the authors propose a heuristic method: for each query $Q_i$, they sample $\Theta(\log n)$ scalar products of the form $\langle Q_i, K_j \rangle$ and approximate the sparsity measurement of the query $Q_i$ using these samples. Let us denote by $\operatorname{ProbSelect}(Q, K)$ the method that selects the top $u$ queries of $Q$ with this sampling heuristic and masks the rest of the queries to zeroes. Then, the probabilistic attention mechanism is 
\begin{equation*}
    \operatorname{ProbAttn(Q, K, V)} = \operatorname{softmax}\left( \frac{\operatorname{ProbSelect}(Q, K) K^T}{\sqrt{d_q}}\right) V.
\end{equation*}
The product $\operatorname{ProbSelect}(Q, K) K^T$ does not have to be fully computed. Indeed, for each query $Q_i$ that is not selected, the corresponding attention scores are uniform, that is, $a_{ij} = 1/n$, as $Q_i = 0$. Thus, we only have to carry out $n$ scalar products for each one of the $u = \Theta(\log n)$ rows selected to calculate $\operatorname{ProbAttn(Q, K, V)}$, leading to $\Theta(n \log n)$ complexity. 

Even though the ProbAttn method does significantly reduce the time and memory complexity of the attention mechanism and seems to be able to capture the most significant attention scores in practice according to the experiments carried out in~\citep{zhou2023informer}, it does lead to a significant loss of information. Indeed, the output of $\operatorname{ProbAttn(Q, K, V)}$ has $n - u$ entries that are equal to the average of the values. We believe this to be a major drawback of this attention mechanism, and this is corroborated in our experimentation. The authors of~\citep{zhou2023informer} try to fix this issue by applying residual connections in their model, as well as a complex architecture that exploits convolutions in conjunction with ProbAttn. However, our experiments show that this is not enough, as the vanilla transformer is able to significantly outperform Informer.

\subsection{The Encoder-Decoder architecture and Positional Encoding} \label{sec:trans:encoder-decoder}

The encoder-decoder architecture in transformers forms the backbone of many advanced models in natural language processing and beyond. In this architecture, the encoder processes the input sequence, transforming it into a rich, contextual representation. Here we briefly describe the architecture proposed in~\citep{vaswani2017attention}, which we use in our model in conjunction with \proposal. Each encoder layer consists of a multi-head self-attention mechanism, which allows the model to weigh the importance of different parts of the input relative to each other, and two linear projections acting on the dimension of the variables (i.e. the last dimension of the input) with activation functions. Formally, these linear projections act as the matrix $W^O$ in the definition of $\operatorname{MultiHead}$, they are represented as $d_{model}\times d_{model}$ matrices, but they have LeakyReLu activation functions on the outputs. The purpose of projection layers is to learn a domain representation that boosts the performance of the attention mechanism. 

The decoder layers are similar to the encoder layers. The main difference is that they apply two multi-head self-attention layers. In the first multi-head self-attention, the queries, keys and values are set to the input of the layer, whereas in the second one the keys and queries are set to the output of the encoder and the values to the input of the layer. This is illustrated in Figure~\ref{fig:vanilla-transformer}. This second attention layer that helps the decoder to focus on relevant parts of the input sequence, enhancing its ability to produce contextually relevant outputs. We highlight the presence of residual connections in the encoder and decoder layers. These residual connections consist in adding the input of the attention mechanism its the output, reducing the loss of positional information. 

Once the encoder and decoder layers are defined, the architecture of the vanilla transformer is relatively simple: the encoder consists of $N$ consecutive applications of the encoder layer and the decoder consists of $N$ consecutive applications of the decoder layer. The value of $N$ can be picked in each application, and will be provided it in each one of our experiments. We note that increasing $N$ does not significantly change the number of parameters of the model, as the feed forward layers applied in the encoder and decoder layers are linear projections, as described above, which have $d_{model}^2$ parameters. That is, the number of parameters in the encoder and the decoder layers is $\Theta(d_{model}^2)$.

Finally, the output of the last decoder layer goes through a linear layer. To reduce the number of parameters, we have implemented this layer as a projection on the time dimension, that is, if the input of the transformer has $n$ elements of the time series, this linear layer corresponds to a matrix $n \times m$ that $m$ computes linear combinations of the $n$ points produced in the last decoding layer, where $m$ is the number of points that we want to forecast. This is the output of our model, as we do not need to perform a softmax layer to compute probabilities. 

\begin{figure}[!hbt]
    \centering
    \includegraphics[scale=0.5]{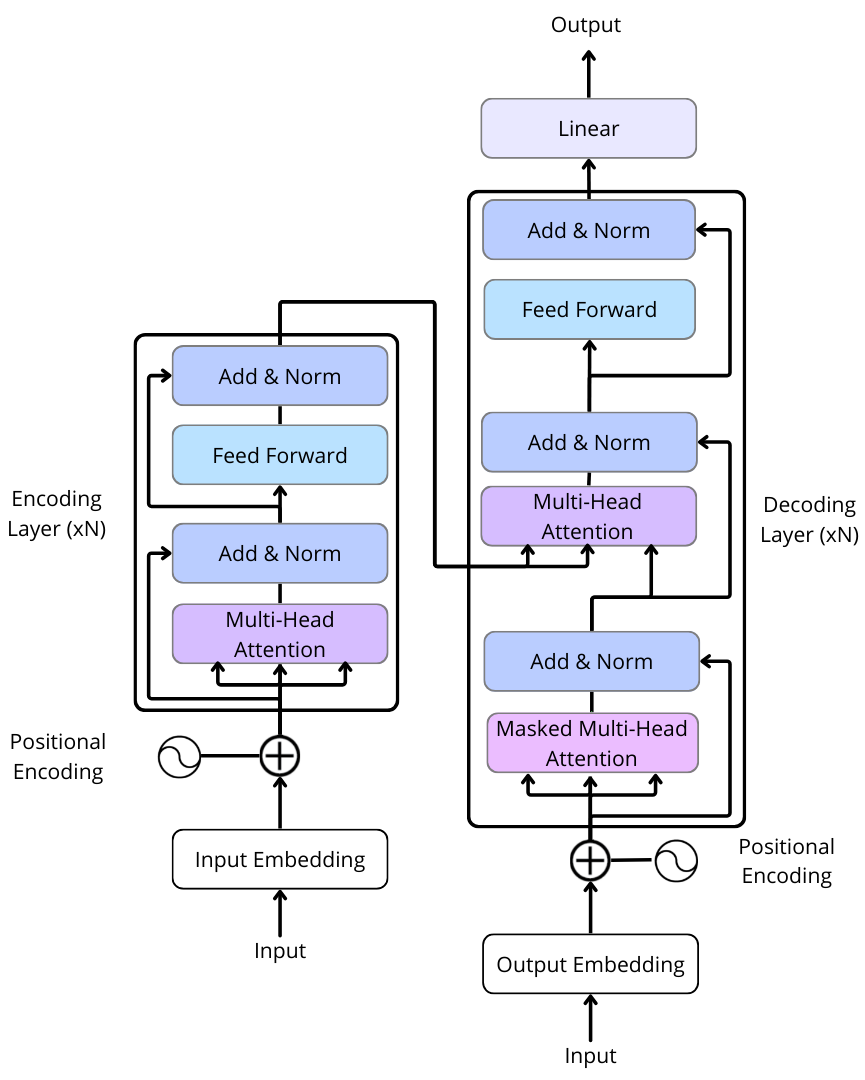}
    \caption{Transformer architecture used for LAM.}
    \label{fig:vanilla-transformer}
\end{figure}

Positional encoding in transformers is essential for providing sequence information to the model, which inherently lacks an understanding of order due to its self-attention mechanism (see Proposition~\ref{prop:perm}). Positional encoding involves adding a positional tensor to the input sequence, which is usually derived from sine and cosine functions, to indicate each entry relative or absolute position. In our experiments, we add to the input matrix the positional encoding employed in the vanilla model,
\begin{equation*}
    \operatorname{PE}(i,j) = \sin(i / 10000^{j / d_{model}}) + \cos(i / 10000^{j / d_{model}}).
\end{equation*}

\section{\proposalextended}
\label{sec:local-attention}

In this section we introduce \proposal\ and an efficient algorithm, based on expressing \proposal\ in tensor algebra, that leads to $\Theta(n \log n)$ memory and time complexity. This algorithm has been implemented in stock Pytorch and TensorFlow without the need for any low level CUDA code. 

\subsection{Theoretical definition}
\label{sec:local-attention:subsec:theoretical}

The high-level fundamental behind \proposal\ is exploiting locality in order to re-define the attention mechanism. Previous studies on time series have empirically shown that the rows of the attention scores matrix behave as long tail distributions~\citep{li2019enhancing, zhou2023informer}. This behaviour is present at \textit{any} layer independently of the depth. For example, the coefficients $a_{n-1,0}, a_{n-1,1}, \ldots, a_{n-1,n-1}$ of the last row of the score matrix are mostly increasing, with most of the probability mass function being allocated to the last few coefficients, which contrasts to the behaviour of the attention mechanism on language related tasks~\citep{galassi2020attention}. Even though both NLP and time series forecasting deal with sequences of data, the main difference relies on the nature of this data. In NLP words are embedded into a vector space so that words related in meaning have similar dot product. Moreover, words with related meaning may come up far apart in a paragraph whereas consecutive words may have small dot product. On the other hand, most variables in time series problems are continuous and, thus, data tends to be extremely similar around a particular timestamp, hence leading to the larger attention scores in that region. This motivates the following working hypothesis for this research.
\begin{hyp} \label{hyp:time-series}
    Let $X_t$ be a (mostly) continuous time series. On inputs $x_{t-n+1}, \ldots, x_{t}$, for any $j \in \{0,1, \ldots, n-1\}$, the attention scores of the row corresponding to $x_{t-n+1+i}$, denoted $a_{i,0}, a_{i,1}, \ldots, a_{i,n-1}$, are mostly concentrated around $a_{j,j}$.
\end{hyp}

Previous attempts have exploited this sparsity properties of self-attention in time series modelling to design selective masking strategies that, for each row of the attention scores matrix, set most entries to zero. For example, the LogSparse Transformer~\citep{li2019enhancing} (denoted LogTrans in our experimentation) uses, for each key with timestamp $t_0$, attention scores with those queries whose timestamp $t$ is in a local neighbourhood of $t_0$ and those queries whose timestamp $t$ satisfies $t_0 - t = 2^j$ for some positive integer $j$. The LongFormer~\citep{beltagy2020longformer} extends the above work to a more complicated sparse configuration in the realm of NLP. These masking strategies partially resolve the loss of positional information problem of the attention mechanism as the input is only partially combined at each attention layer and, thus, Proposition~\ref{prop:perm} does not hold for such mechanisms. However, it is not clear how to obtain an improvement in time and memory complexity, as these attention mechanisms have not been written in tensor algebra in an efficient way. Instead, sparse matrices are used to perform computations and represent the attention matrix, leading to minor memory and computational efficiency improvements. 

In this work, we propose the \proposalextended\ (\proposal), which only computes dot products on nearby inputs of the attention layer. As we will see, \proposal\ can be implemented efficiently in tensor algebra, leading to almost linear time and memory complexity in practice. 

\begin{definition}[Local Attention]\label{def:lam}
Let $L$ be a neighbourhood size (in our experiments $L =  4 \lceil \log n \rceil$) and let $M \in \{0,-\infty\}^{n \times n}$ be the square matrix with $M_{i,j} = 0$ when $i - L+1 \le j \le i$ and $M_{i,j} = -\infty$ otherwise. Then \proposal\ is defined as 

\begin{equation} \label{eq:lam}
    \operatorname{LAM}(Q, K, V) = \operatorname{softmax}\left( \frac{QK^T + M}{\sqrt{d_q}}\right) V.
\end{equation}    
\end{definition}

Note that adding the matrix $M$ and applying softmax leads to several zero entries in the matrix $\operatorname{softmax}\left( (QK^T + M)/\sqrt{d_q}\right)$. In fact, the scores corresponding to the scalar/dot product of the query $Q_i$ with the keys $K_j$ with $i -L+1 \le j \le i$ are the only ones that may not be zero. That is, the $j$-th output of $\operatorname{LAM}(Q, K, V)$ is a weighted average of $V_{j-L+1}, V_{j-L+2}, \ldots ,V_j$. This is visualised in Figure~\ref{fig:example:a}.

At first glance, the brute force implementation of \proposal\ by computing $QK^T$ gives quadratic time and memory complexity. Here we explain how to achieve $\Theta(n \log n)$ complexity in time and memory with a tensor algebra re-formulation of \proposal. For ease of exposition, assume that $L$ divides $n$ (the general case will be addressed at the end of this section). Our algorithm to compute \proposal\ is provided in Algorithm~\ref{alg:local}, although the reader may find the following mathematical description useful. 

Let $s = n/L$. The key idea is splitting $Q$ into $s$ blocks of dimensions $L \times d_q$ and determining the smallest submatrices of $K^T$ that we have to multiply these blocks by so that the end results contain the dot products 
\begin{equation} \label{eq:dot-products:lam}
\begin{aligned}
\langle Q_i, K_j \rangle \text{ for all } & r \in \{0, 1, \ldots, s-1\}, i \in \{rL, r(L+1)-1\} \\ & \text{ and } j \text{ with } i - L +1 \le j \le i. 
\end{aligned}
\end{equation}
Note that these dot products are exactly the ones not masked by $M$ in the definition of \proposal. In our algorithm, we compute $2L-1$ dot products for each row of $Q$ instead of the $L$ dot products described in Eq.~\eqref{eq:dot-products:lam}. The extra number of dot products computed will allow us to write \proposal\ in tensor algebra, without affecting the time and memory complexity. We formally introduce this ``splitting'' idea in Definition~\ref{def:TQ}.

\begin{definition}[The tensor $T_Q$] \label{def:TQ}
    Let $s = n/L$. We define $T_Q$ as the tensor with dimension $s \times L \times d_q$ such that $T_Q[r,i_1,t] = Q[rL+i_1, t]$ for all $r \in \{0,1,\ldots, s-1\}$, $i_1 \in \{0, 1, \ldots, L-1\}$ and $t \in \{0,1,\ldots, d_q-1\}$.
\end{definition}

 Let $r \in \{0,1,\ldots, s-1\}$ be the index of the block of $Q$ under consideration. We claim that it suffices to compute the dot products 
 \begin{equation*}
 \begin{aligned}
     \langle Q_{rL+i_1}, K_j \rangle = \langle T_Q[r,i_1,\cdot], K_j \rangle & \text{ for all }  r \in \{0, 1, \ldots, s-1\},\\ & i_1 \in \{0,1,\ldots, L-1\} \text{ and } \\ &  j \text{ with } (r-1)L+1 \le j \le (r+1)L-1.      
 \end{aligned}
 \end{equation*}
 Indeed, with the change of variables $i = r L + i_1$, using $i_1 \in \{0,1,\ldots, L-1\}$ we find that $rL \le i \le (r+1)L-1$ and, thus,
\begin{equation} \label{eq:inclusion}
    [i - L +1, i] \subseteq [(r-1)L+1, (r+1)L-1].
\end{equation}
This motivates the following definition.
\begin{definition}[The tensor $T_K$] \label{def:TK}
    Let $s = n/L$. We define $T_K$ as the tensor with dimensions $s \times (2L-1) \times d_q$ such that, for each for all $r \in \{0, \ldots, s-1\}$, $j_1 \in \{0, 1, \ldots, 2L-2\}$ and $t \in \{0, 1, \ldots, d_{q}-1\}$, $T_K[r, j_1, t] = K[(r-1)L+1+j_1, t]$ when $(r-1)L+1+j_1 \ge 0$ and $T_K[0, j_1, t] = 0$ otherwise. 
\end{definition}

We note that $(r-1)L+1+j_1 < 0$ only when $j_1 < (1-r)L-1$, which only occurs in the case $r = 0$ and $j_1 \le L-2$. Indeed, for $r = 0$ in Eq.~\eqref{eq:dot-products:lam}, we only have to compute the dot product $\langle Q_{i_1}, K_j\rangle$ for $0\le j \le i_1$. The extra zeroes added in $T_k$ for $r= 0$ allow us to maintain the dimension of $T_K[r, \cdot, \cdot]$ across all values of $r$. Lemma~\ref{lem:TM} allows us to go back and forth between the indexes $(i,j)$ and $(i_1, j_1)$ for the changes of variables $i = rL + i_1$ and $j = (r-1)L + 1 + j_1$. The proof of this lemma can be found in Appendix~\ref{appendix:theoretical-proofs}.

\begin{lemma} \label{lem:TM}
 Let $r \in \{0,1,\ldots,s-1\}$. For each $i_1 \in \{0, 1, \ldots, L-1\}$ and  $j_1 \in \{0, 1, \ldots, 2L-2\}$, let $i = rL+i_1$ and $j = (r-1)L+1+j_1$. We have $i-L+1 \le j \le i$ if and only if $i_1 \le j_1 \le i_1+L-1$.
\end{lemma}

Now we can formally prove that our definition of $T_Q$ and $T_K$ capture all the dot products needed in the definition of \proposal\ (Definition~\ref{def:lam}), which is also visualised in Figure~\ref{fig:example}.

\begin{lemma} \label{lem:TQ-TK}
    For each $r \in \{0, 1, \ldots, s-1\}$, the matrix defined as $T_A[r, \cdot, \cdot] = T_Q[r,\cdot,\cdot] T_K[r,\cdot,\cdot]^T$ has dimensions $L \times (2L-1)$. For each $i \in \{0,1,\ldots, n-1\}$ and $j$ with $i-L+1\le j \le i$, setting $i_1 = i \pmod L$, $r = (i -i_1)/L$ and $j_1 = j-(r-1)L-1$, we have $r \in \{0, 1, \ldots, s-1\}$, $i_1 \in \{0, 1, \ldots, L-1\}$ and $j_1 \in \{0,1,\ldots,2L-2\}$. Moreover, the equality $\langle Q_i, K_j \rangle = T_A[r, i_1, j_1]$ holds.
\end{lemma}

\begin{figure}[H]
    \centering
    \begin{subfigure}{.48\textwidth} 
        \includegraphics[width=\textwidth]{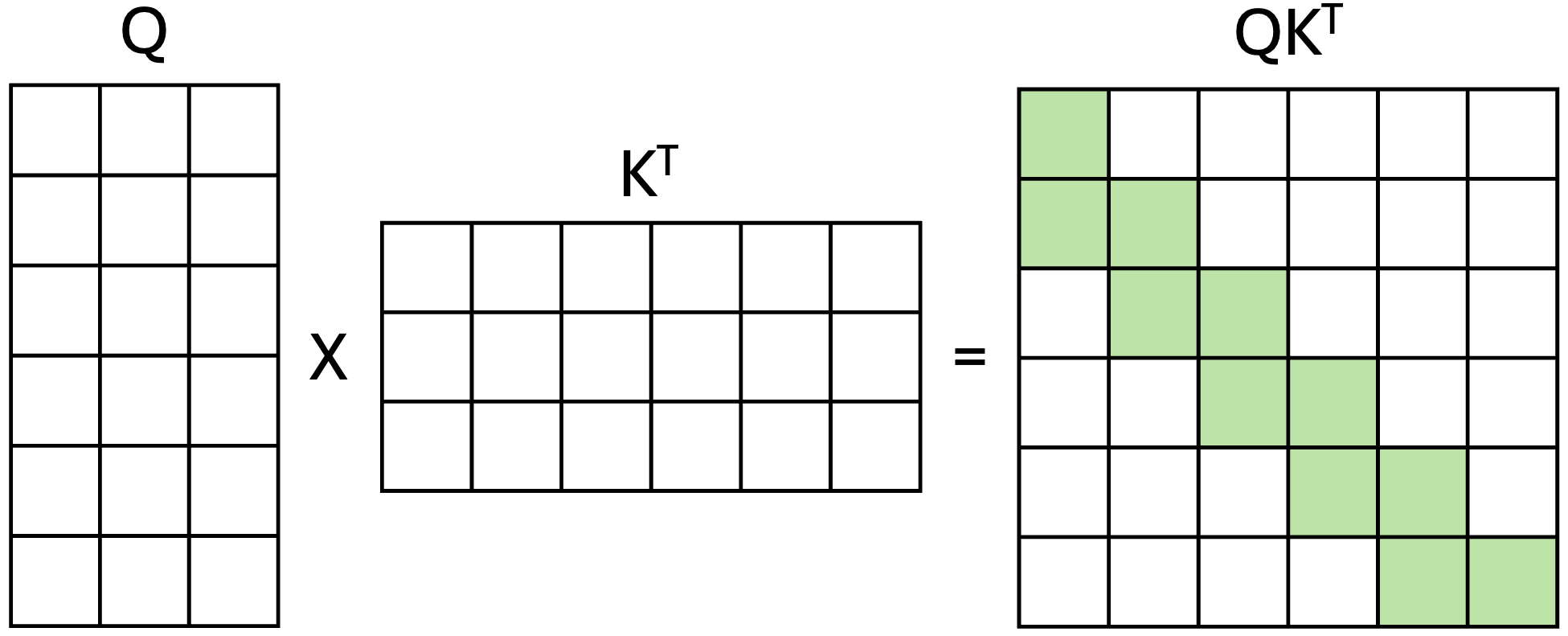}    
        \caption{Visualisation of the product $Q K^T$. In green, we highlight the dot products needed in \proposal, as the rest are set to $-\infty$ when adding $M$ in Eq.~\eqref{eq:lam}.}
        \label{fig:example:a}
    \end{subfigure}
    \hfill
    \begin{subfigure}{.48\textwidth}
        \includegraphics[width=\textwidth]{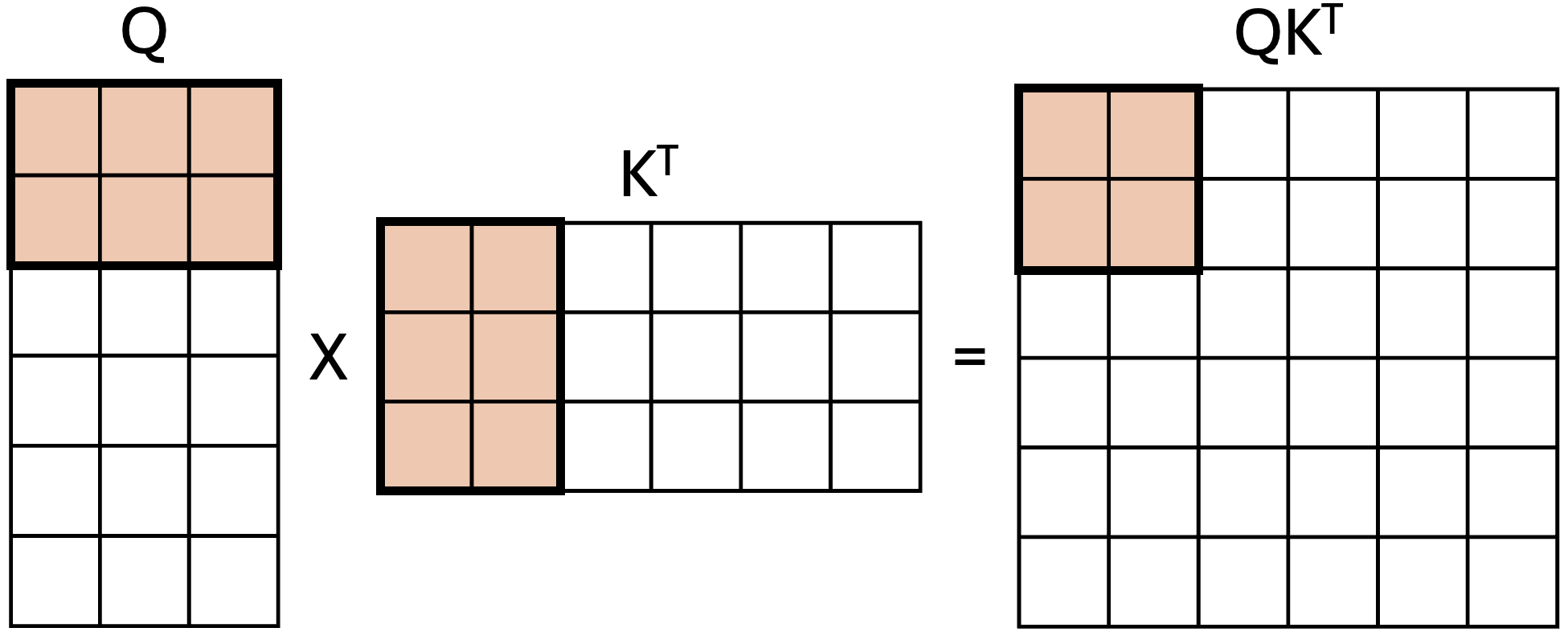}
        \caption{In light brown we have the block $T_Q[0, \cdot, \cdot]$ of $Q$, the non-zero elements of the block $T_K[0, \cdot, \cdot]^T$ of $K^T$, and the dot products corresponding to $T_Q[0,\cdot,\cdot] T_K[0,\cdot,\cdot]^T$.}
        \label{fig:example:b}
    \end{subfigure}
    \begin{subfigure}{.48\textwidth}
        \includegraphics[width=\textwidth]{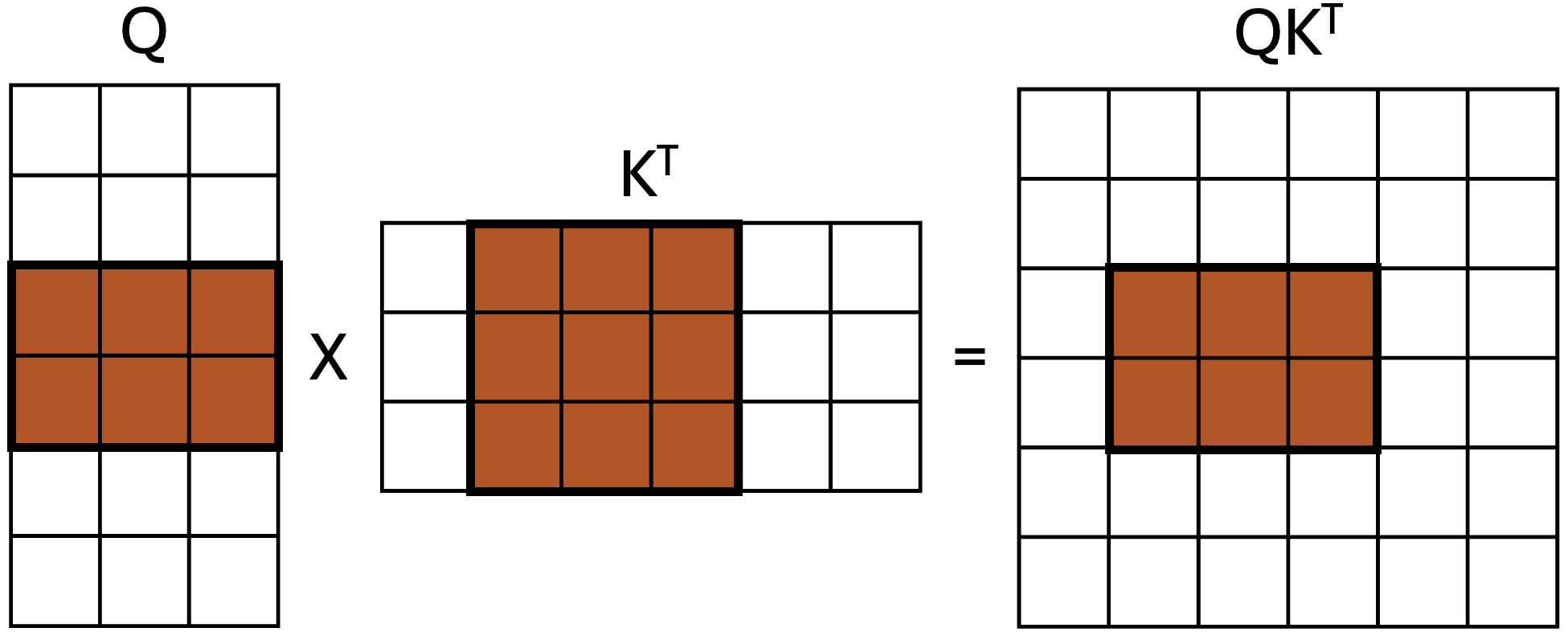}
        \caption{In brown we have the block $T_Q[1, \cdot, \cdot]$ of $Q$, the block $T_K[1, \cdot, \cdot]^T$ of $K^T$, and the dot products corresponding to  $T_Q[1,\cdot,\cdot] T_K[1,\cdot,\cdot]^T$.}
        \label{fig:example:c}
    \end{subfigure}
    \hfill
    \begin{subfigure}{.48\textwidth}
        \includegraphics[width=\textwidth]{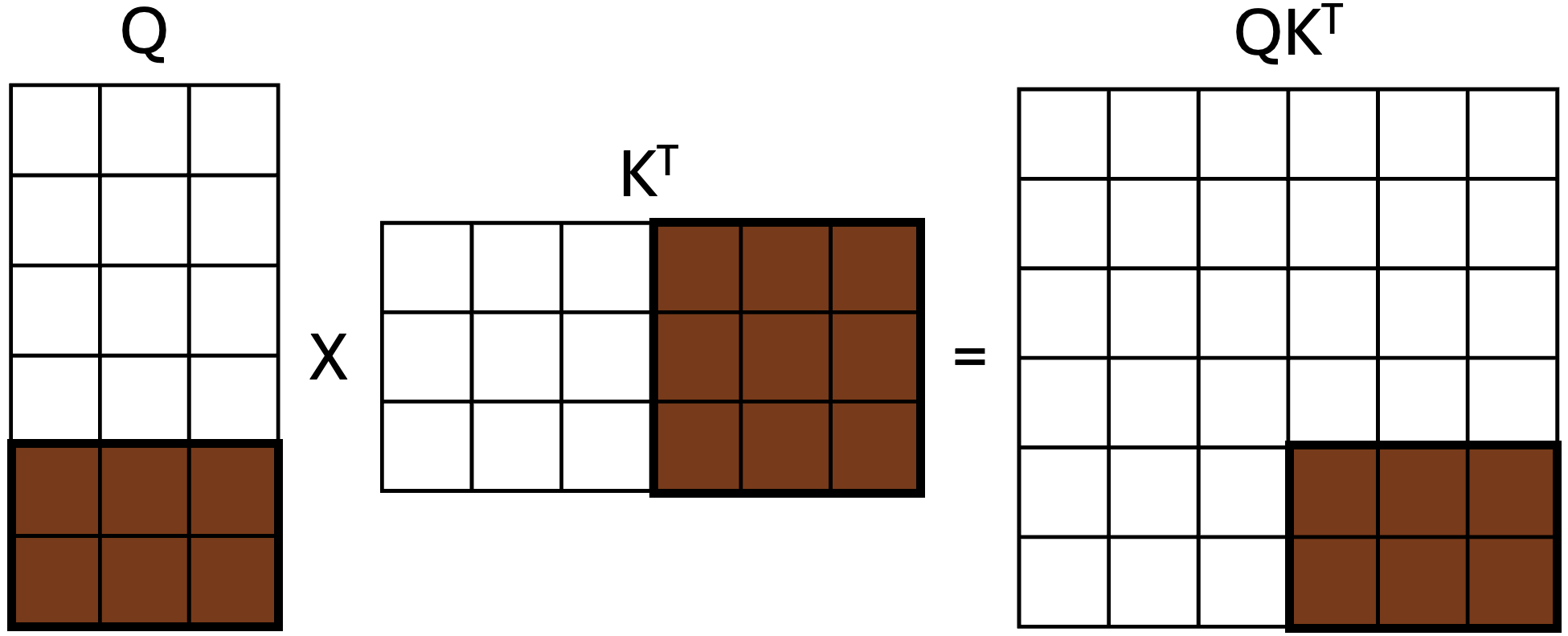}
        \caption{In dark brown we have the block $T_Q[2, \cdot, \cdot]$ of $Q$, the block $T_K[2, \cdot, \cdot]^T$ of $K^T$, and the dot products corresponding to  $T_Q[2,\cdot,\cdot] T_K[2,\cdot,\cdot]^T$.}
        \label{fig:example:d}
    \end{subfigure}
    \caption{Example of our algorithm to compute the dot products of $QK^T$ needed in \proposal\ for $n = 6$, $d_{q} = 3$ and $L=2$. We note that $s = n / L = 3$. The dot products needed in \proposal\ are highlighted in green in (a), the dot products computed in our algorithm are highlighted in a shade of brown in $Q K^T$ in Subfigures (b), (c), (d). We note that every dot product highlighted in green in (a) is in brown in either (b), (c) or (d), as Lemma~\ref{lem:TQ-TK} claims.}
    \label{fig:example}
\end{figure}

At this point we describe how we express $\operatorname{softmax}\left( (QK^T + M)/ \sqrt{d_q}\right)$ in terms of $T_A[r, \cdot, \cdot] = T_Q[r,\cdot,\cdot] T_K[r,\cdot,\cdot]^T$. First, we mask the undesired dot products using a tensor of dimension $s \times L \times (2L-2)$ that emulates the matrix $M$ of Definition~\ref{def:lam}. 

\begin{definition}[The tensor $T_M$] \label{def:TM}
    Let $s = n/L$. We define $T_M$ as the tensor with dimension $s \times L \times (2L-2)$ such that $T_M[r,i_1,j_1] = 0$ when $i_1 \le j_1 \le i_1+L-1$ and $T_M[r,i_1,j_1] = -\infty$ otherwise.
\end{definition}

In view of Lemma~\ref{lem:TM}, adding $T_M$ to $T_A$ acts as the same masking as adding $M$ to $QK^T$, we are just operating on the variables $(i_1, j_1)$ instead of $(i,j)$. Figure~\ref{fig:adding-M} visualises the computation $\operatorname{softmax}(T_A+T_M)$.

\begin{figure}[H]
    \centering
    \begin{subfigure}{\textwidth} 
        \includegraphics[width=\textwidth]{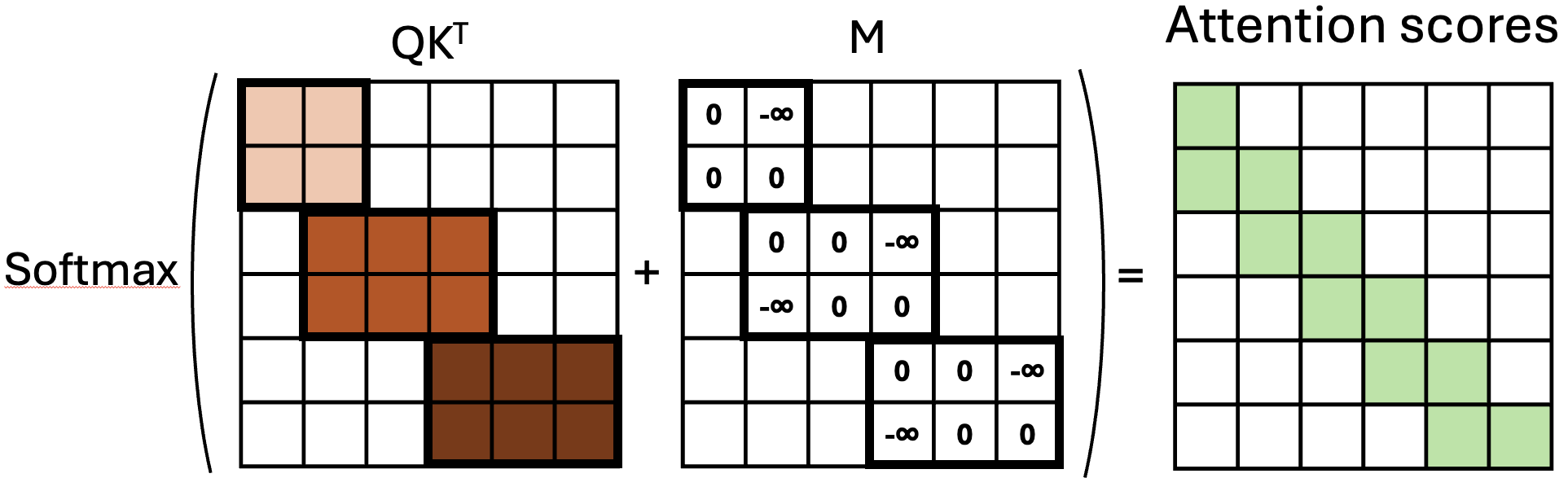}
    \end{subfigure}
    \caption{Visualisation of computation of the attention scores for \proposal. For the matrix $QK^T$ we show in brown the dot products computed in Figure~\ref{fig:example} (that is, those dot products computed for $T_A$), where light brown corresponds to Figure~\ref{fig:example:b}, medium brown corresponds to Figure~\ref{fig:example:c} and dark brown corresponds to Figure~\ref{fig:example:d}. The attention scores of \proposal\ are computed after adding $T_M$ (the sub-blocks of $M$ given in Definition~\ref{def:TM}), to $T_A$ and applying the Softmax function to each row. 
    In green, we highlight the attention scores computed by \proposal.}
    \label{fig:adding-M}
\end{figure}

We can now prove that $T_S = \operatorname{softmax}((T_Q T_K^T + T_M) / \sqrt{d_q})$ contains the attention scores of $\operatorname{softmax}\left( (QK^T + M)/\sqrt{d_q}\right)$, other than the extra zeroes that we have not computed.

\begin{lemma} \label{lem:T_S}
    The tensor defined as $T_S = \operatorname{softmax}((T_Q T_K^T + T_M) / \sqrt{d_q})$, where $\operatorname{softmax}$ is applied on the last dimension, has dimensions $s \times L \times (2L-1)$. Let $S$ be the matrix $\operatorname{softmax}\left( (QK^T + M)/\sqrt{d_q}\right)$. For each  $i \in \{0,1,\ldots, n-1\}$ and $j$ with $i-L+1\le j \le i$, setting $i_1 = i \pmod L$, $r = (i -i_1)/L$ and $j_1 = j-(r-1)L-1$, we have $S[i,j] = T_S[r, i_1, j_1]$. Moreover, all the other entries of $S$ and $T_A$ are zero.
\end{lemma}

Up to now, we have managed to determine the attention scores of \proposal\ while avoiding to compute all the entries of the product $QK^T$. It remains to multiply the attention scores matrix by the values tensor $V$. This operation can be similarly performed by extracting blocks of $V$, which motivates Definition~\ref{def:TV}, and can be visualised in Figure~\ref{fig:example-2}. 

\begin{definition}[The tensor $T_V$] \label{def:TV}
    Let $s = n/L$. We define $T_V$ as the tensor with dimensions $s \times (2L-1) \times d_{model}$ such that, for each for all $r \in \{0, \ldots, s-1\}$, $j_1 \in \{0, 1, \ldots, 2L-2\}$ and $t \in \{0, 1, \ldots, d_{model}-1\}$, $T_V[r, j_1, t] = V[(r-1)L+1+j_1, t]$ when $(r-1)L+1+j_1 \ge 0$ and $T_V[0, j_1, t] = 0$ otherwise. 
\end{definition}

\begin{figure}[H]
    \centering
    \begin{subfigure}{.48\textwidth}
        \includegraphics[width=\textwidth]{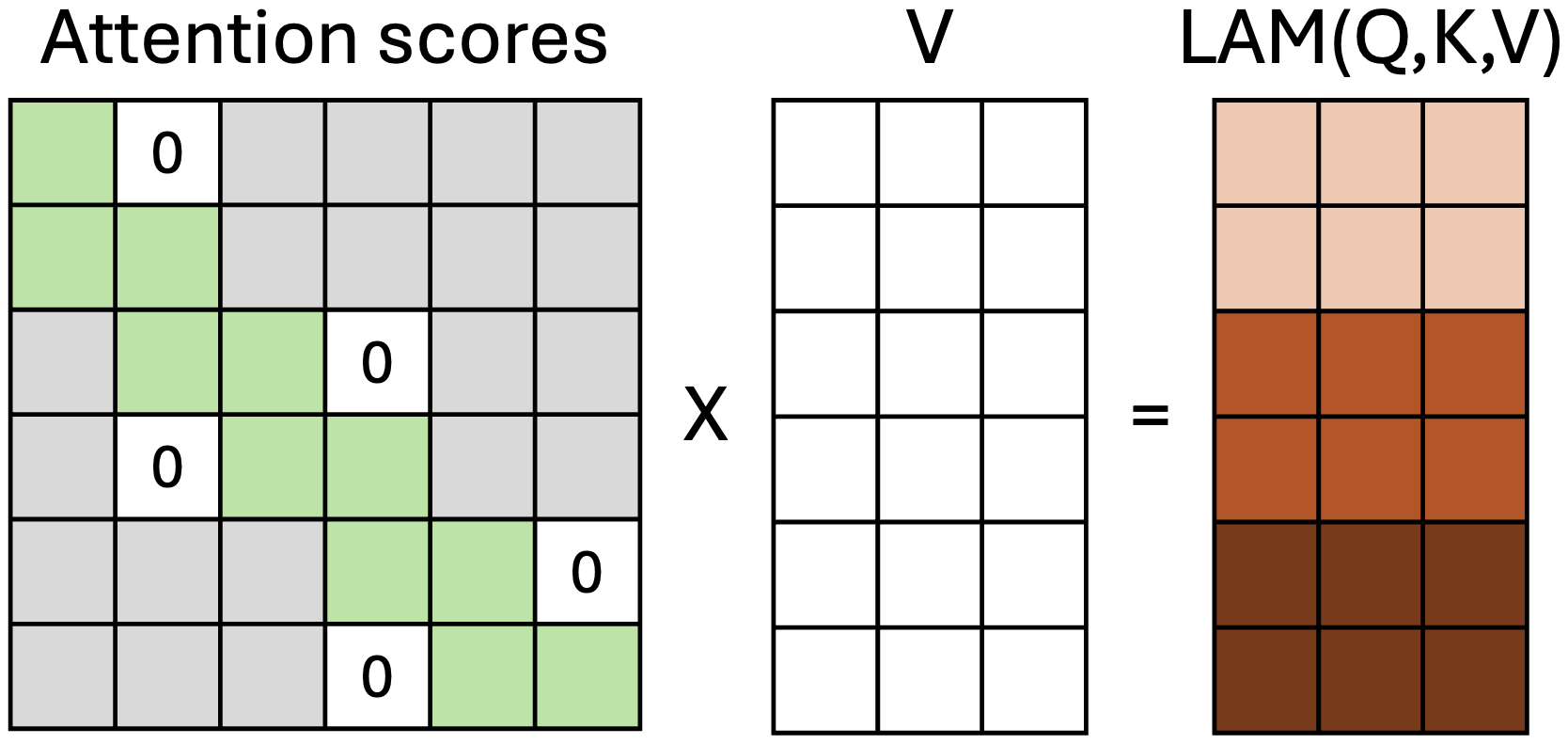}
        \caption{Visualisation of the product $A \cdot V$. In green we have the attention scores computed in $T_S$, which correspond to the non-zero attention scores of $S$ (Lemma~\ref{lem:T_S}).}
        \label{fig:example-2:a}
    \end{subfigure}
    \hfill
    \begin{subfigure}{.48\textwidth} 
        \includegraphics[width=\textwidth]{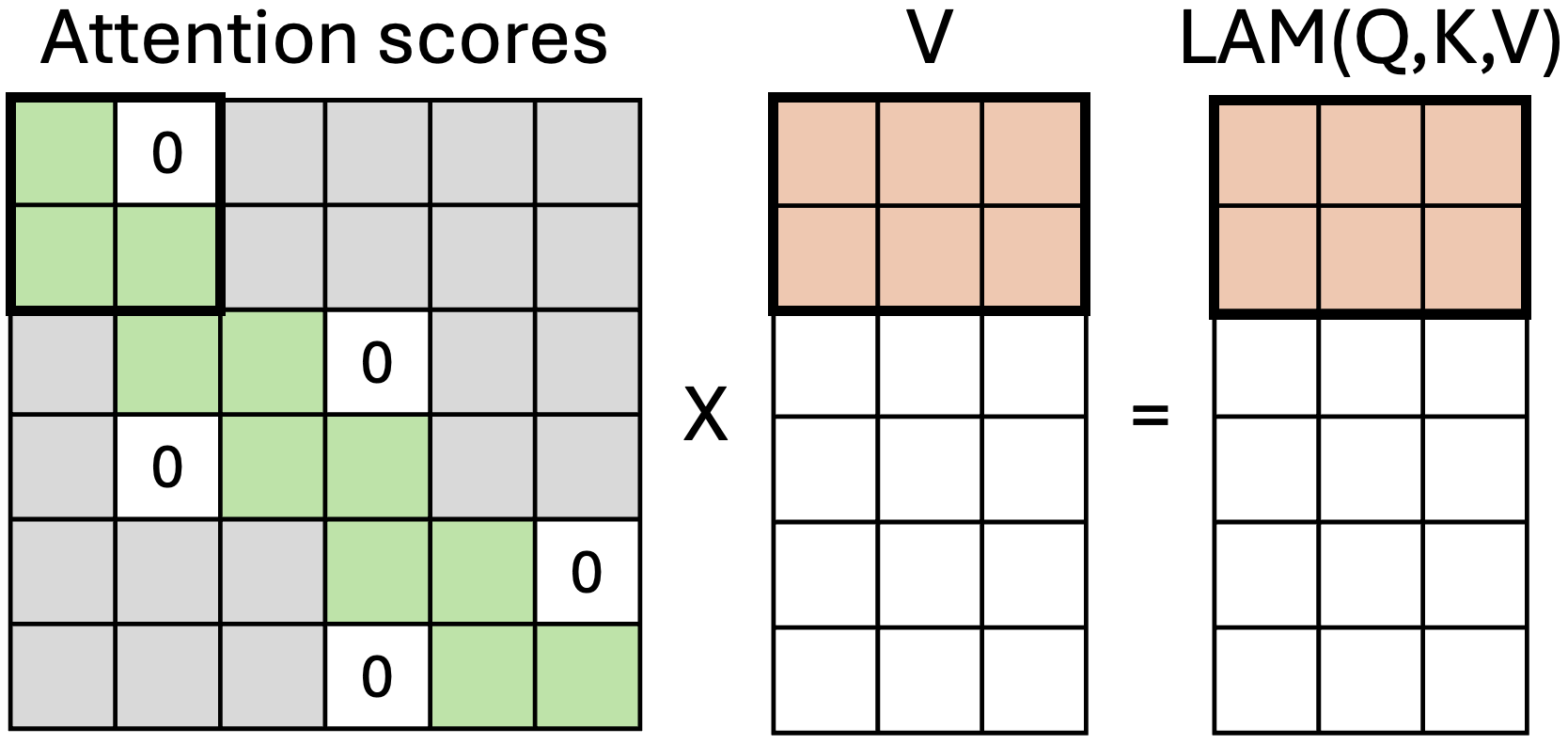}    
        \caption{In light brown we have the block $T_V[0, \cdot, \cdot]$ and the block $\proposal(Q,K,V)[0:L, \cdot]$. We note that $\proposal(Q,K,V)[0:L, \cdot]$ equals the product $\qquad$ $A[0:L,0:L] T_V[0, \cdot, \cdot]$.}
        \label{fig:example-2:b}
    \end{subfigure}
    \begin{subfigure}{.48\textwidth}
        \includegraphics[width=\textwidth]{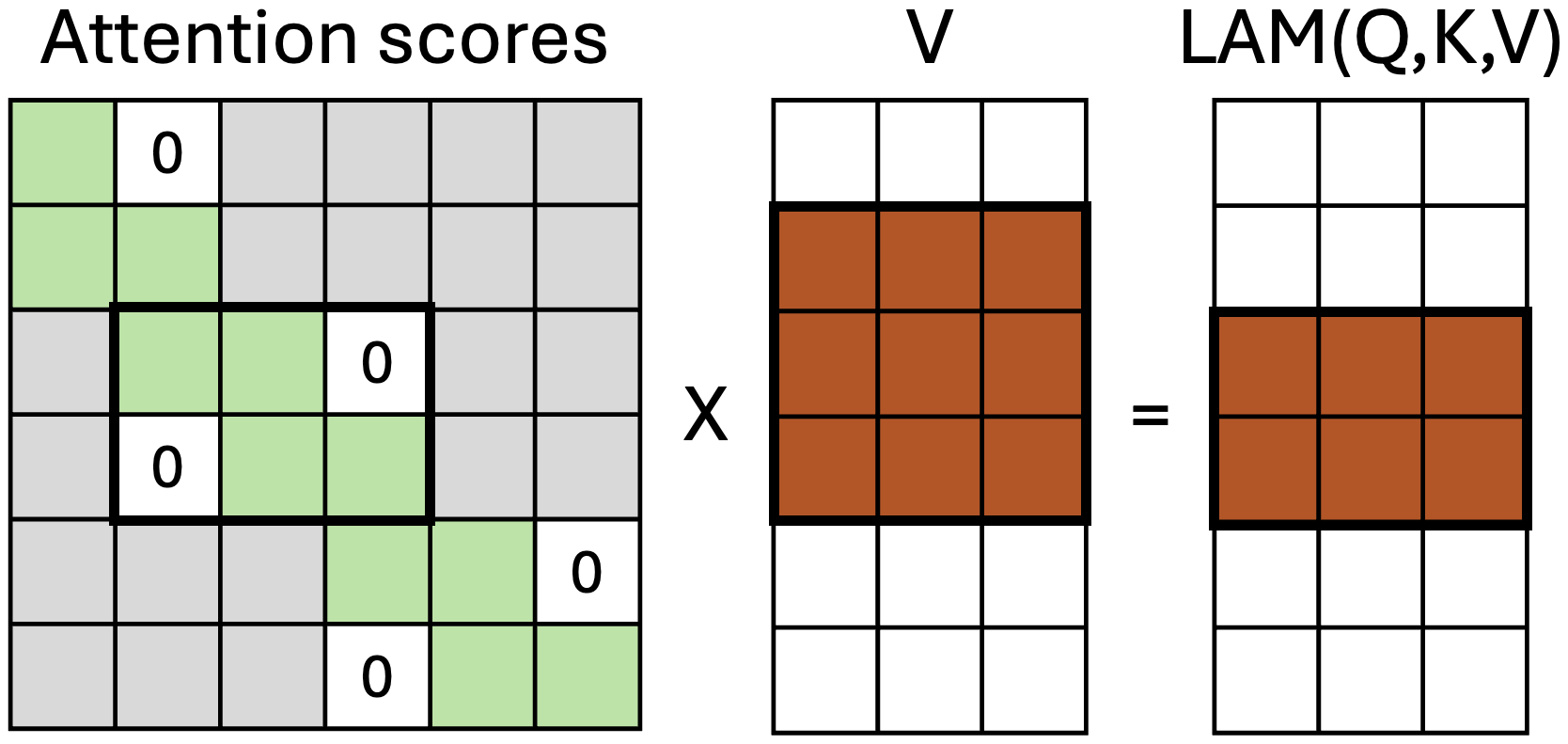}
        \caption{In brown we have the block $T_V[1, \cdot, \cdot]$ and the block $\proposal(Q,K,V)[L:2L, \cdot]$. We note that $\proposal(Q,K,V)[L:2L, \cdot]$ equals the product $\quad$ $A[L:2L,L-1:2L] T_V[1, \cdot, \cdot]$.}
        \label{fig:example-2:c}
    \end{subfigure}
    \hfill
    \begin{subfigure}{.48\textwidth}
        \includegraphics[width=\textwidth]{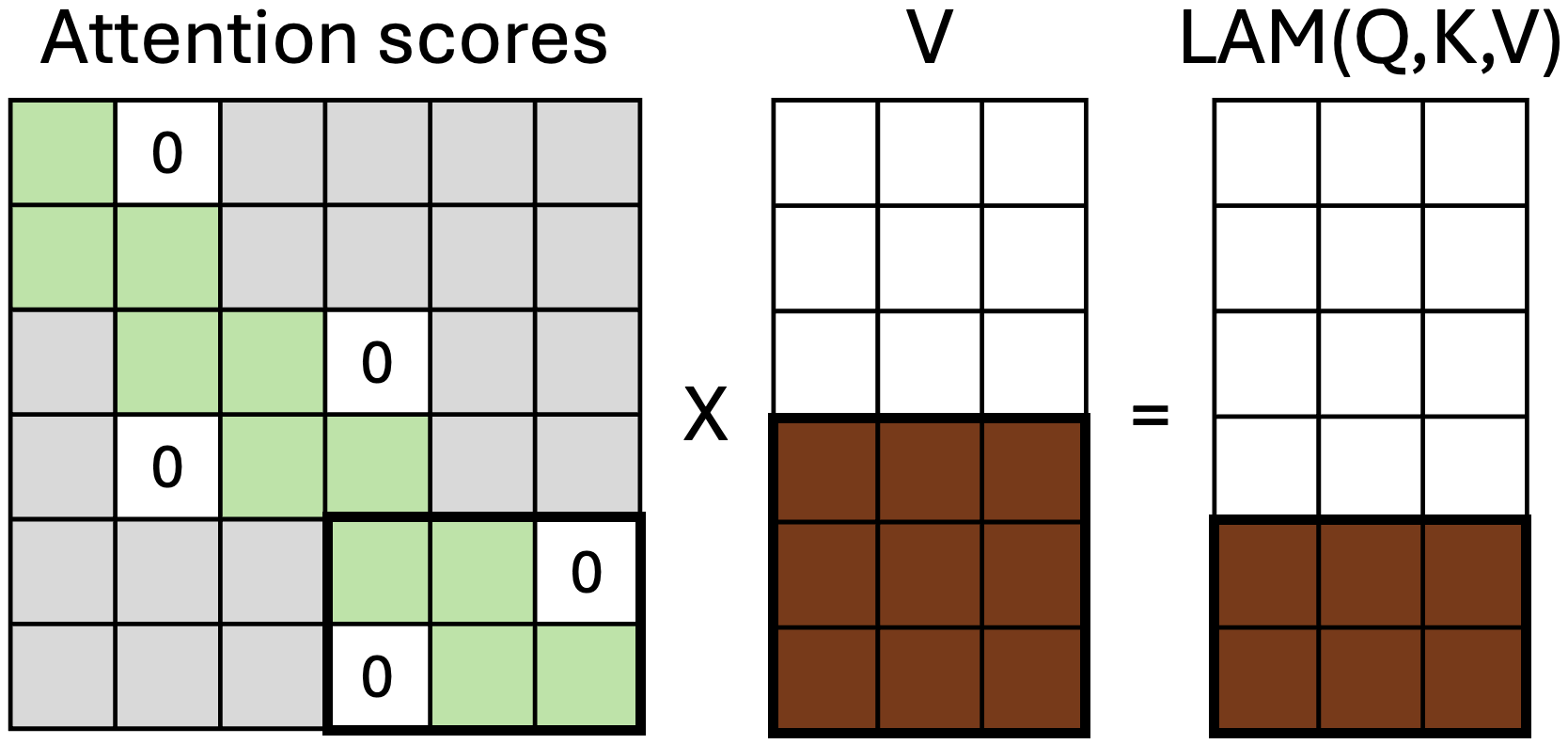}
        \caption{In dark brown we have the blocks $T_V[2, \cdot, \cdot]$ and $\proposal(Q,K,V)[2L:3L, \cdot]$. We note $\proposal(Q,K,V)[2L:3L, \cdot]$ equals the product $\quad$ $A[2L:3L,2L-1:3L] T_V[1, \cdot, \cdot]$.}
        \label{fig:example-2:d}
    \end{subfigure}
    \caption{Example of our algorithm to compute the dot products of $S V$ needed in \proposal\ for $n = 6$, $d_{model} = 3$ and $L=2$, where $S = \operatorname{softmax}\left( (QK^T + M)/\sqrt{d_q}\right)$. We note that $s = n / L = 3$. The values tensor is divided $s$ in blocks, with a different shadow of brown, in Subfigures (b), (c), (d).}
    \label{fig:example-2}
\end{figure}

\begin{lemma} \label{lem:T_V}
    The tensor defined as $T_{LAM} = \operatorname{softmax}((T_Q T_K^T + T_M) / \sqrt{d_q})T_V$, where $\operatorname{softmax}$ is applied on the last dimension, has dimensions $s \times L \times d_{model}$. For each $i \in \{0,1,\ldots, n-1\}$, setting $i_1 = i \pmod L$ and $r = (i -i_1)/L$, we have $\proposal(Q,K,V)[i, \cdot] = T_{LAM}[r, i_1, \cdot]$.
\end{lemma}

\subsection{\proposal\ implementation}
\label{sec:local-attention:subsec:implementation}

Our algorithm to compute $\proposal(Q,K,V)$ can now be inferred from Lemma~\ref{lem:T_V} and is explicitly given in Algorithm~\ref{alg:local}.

\begin{algorithm}[H]
  \begin{algorithmic}[1]

    \caption{Efficient implementation of \proposal\  with tensor algebra}  \label{alg:local}

    \REQUIRE The query and key matrices $Q$ and $K$, and the value matrix $V$.

    \ENSURE $Q$ and $K$ have dimensions $n \times d_q$ and $V$ has dimensions $n \times d_v$. Let $s = n / L$. Let $T_M$ be the tensor with dimensions $s \times L \times(2 L-1)$ such that $T_M[r, i_1, j_1] = 0$ when $i_1 \le j_1 \le i_1+L-1$ and $T_M[r, i_1, j_1] = -\infty$ otherwise. Here  the tensor $T_M$ is a local mask, that will be applied to each one of the $s$ ''splits" or blocks of $QK^T$ that we compute, emulating the work of the mask $M$ in Eq.~\eqref{eq:lam}. This tensor is pre-computed.
    
    \STATE Let $T_Q$ be the tensor with dimension $s \times L \times d_q$ such that $T_Q[r, i_1, t] = Q[rL+i_1, t]$.

    \STATE Let $T_K$ be the tensor with dimension $s \times (2L-1) \times d_q$ such that $T_K[r, j_1, t] = 0$ when $(r-1)L+1+j_1 < 0$ and $T_K[r, j_1, t] = K[(r-1)L+1+j_1, t]$ otherwise.

    \STATE Let $T_V$ be the tensor with dimension $s \times (2L-1) \times d_{model}$ such that $T_V[r, j_1, t] = -\infty$ when $(r-1)L+1+j_1 < 0$ and $T_V[r, j_1, t] = V[(r-1)L+1+j_1, t]$ otherwise.

    \STATE $T_A = T_Q T_K^T$. Note that $T_A$ has dimensions $s \times L \times (2L-1)$ and that the matrix multiplication is on the last two dimensions of the tensor.

    \STATE $T_S = \operatorname{SoftMax}(T_A + T_M)$.

    \STATE We compute the tensor $T_R = T_S \cdot T_V$, which has dimensions $s \times L \times d_{model}$.

    \RETURN Concatenate $T_R[0, \cdot, \cdot], \ldots, T_R[s-1, \cdot, \cdot]$ on the dimension corresponding to $L$, obtaining a tensor with dimensions $n \times d_{model}$. This is the output of the algorithm.
\end{algorithmic}
\end{algorithm}

The computation of $T_Q$, $T_K$ and $T_V$ can be performed in linear time on the sizes of $Q, K$ and $V$. In order to do so, we pre-compute a tensor with the indexes that we want to extract from $Q, K$ and $V$ for each one of the $s$ blocks of $T_Q, T_K$ and $T_V$. This allows us to implement Algorithm~\ref{alg:local} in any environment for deep learning based on tensor algebra, see the repository \footnote{https://github.com/ari-dasci/S-LAM} for our Pytorch implementation. The case when $L$ does not divide $n$ is analogous, with the difference that the block $s-1$ of $Q$, $K$ and $V$ is treated separately in the code to account for the extra number of entries. That is, we use Algorithm~\ref{alg:local} to compute the first $(s-1)L$ rows of $\proposal(Q,K,V)$, and compute the last $n -(s-1)L$ rows of $\proposal(Q,K,V)$ as a separate tensor computation. As this extra rows are most $L-1$, the complexity of the algorithm does not change. We refer to our code for details.

We can now provide the main theoretical result of this work on combining the lemmas proved in this section and counting how many steps are performed in Algorithm~\ref{alg:local}. All proofs from this section are available in Appendix~\ref{appendix:theoretical-proofs}.

\begin{thm} \label{thm:main}
    Algorithm 1 computes $\operatorname{\proposal}(Q, K, V)$ in time and memory $\Theta(n L)$, where time complexity is measured in the number of dot products of vectors performed.
\end{thm}

\begin{cor}
    For $L = \lceil 4\log n \rceil$, Algorithm 1 computes $\operatorname{\proposal}(Q, K, V)$ in time and memory $\Theta(n \log n)$.
\end{cor}

\section{\proposal\ empirical comparison to the state of the art}
\label{sec:informer-experimental-setup}

In this section, we present the experimental setup and results yielded by the vanilla transformer, fuelled by \proposal, and compare these to the state of the art in time-series forecasting via transformers. In this section we follow  the evaluation framework proposed in the Informer~\citep{zhou2023informer} paper. We will begin with a brief explanation of the datasets utilised and the algorithms involved in Section~\ref{sec:informer-experimental-setup:setup}, moreover, we will provide details regarding the hardware used in our experiments and the open-access repository where the code is published. In Section~\ref{sec:informer-experimental-setup:results}, we analyse the results obtained including a Bayes test analysis. In Section~\ref{sec:informer-experimental-setup:ProbAttention} we perform a comparison between the attention mechanisms \proposal\ and ProbAttn utilising the vanilla transformer architecture. Finally, in Section~\ref{sec:informer-experimental-setup:conclusions} we argue that the datasets that are used in the literature for benchmarking of models for the LSTF problem are not sufficient, which motivates Section~\ref{sec:lstf-experimental-setup}.

\subsection{Experimental Details} \label{sec:informer-experimental-setup:setup}

\textbf{Datasets:} For the sake of fairness we first follow the experimental methodology of Informer~\citep{zhou2023informer}, we use four subsets of ETT: ETTh1 and ETTh2 (with hourly sampling), and ETTm1 and ETTm2 (with a 15-minute granularity). The ECL dataset records electricity consumption data for 321 clients, converted into hourly consumption, with 'MT 320' as the target variable. The Weather dataset consists of climatological data from approximately 1,600 locations in the United States, collected hourly over four years, featuring 'wet bulb' as the target variable along with 11 other climate features. Each dataset was divided into training, validation, and test sets. Details of the datasets can be observed in Table~\ref{tab:datasets-informer-description}.

\begin{table}[!hbt]
\centering
\begin{tabular}{llll}
\toprule
\textbf{Dataset} & \textbf{Instances} & \textbf{Features} & \textbf{Sampling rate}   \\
\midrule
ETTh1   & 8,640     & 7        & 1 hour     \\
ETTh2   & 8,640     & 7        & 1 hour     \\
ETTm1   & 34,560    & 7        & 15 minutes \\
ETTm2   & 34,560    & 7        & 15 minutes \\
Weather & 24,544    & 12       & 1 hour     \\
ECL     & 18,412    & 321      & 1 hour    \\
\bottomrule
\end{tabular}
\caption{Datasets description. Number of instances, features and sampling rates are displayed in the table.}
\label{tab:datasets-informer-description}
\end{table}

\textbf{Multivariate prediction, error metrics and validation:} The primary objective of this comparison is to determine which model produces the least error in multivariate predictions on the datasets described above. Multivariate prediction involves using a fixed-size input sequence to predict a fixed-size output sequence across all variables in the dataset. The problem of multi-variable forecasting is significantly more challenging than single-variable forecasting. Incorporating more than one input variable allows for the extraction of more valuable information by capturing dependencies between variables. However, forecasting multiple variables simultaneously requires more sophisticated modelling, as errors can accumulate across each component. Furthermore, including more variables increases the model's complexity, posing a challenge to assigning meaningful values to the weights. Consequently, the results obtained from multi-variable forecasting carry more significance and provide a more comprehensive evaluation of model performance. To ensure the validity and representativeness of the results, we have incorporated Bayesian tests~\citep{benavoli2014bayesian}, which serve to facilitate and formalise the evaluation process.

Data has been standardised to have a mean of zero and a standard deviation of one. Errors were assessed using reconstruction error, with Mean Absolute Error (MAE) and Mean Squared Error (MSE) serving as the evaluation metrics.

Given the long training times needed in the LSTF problem, we present the results of one training execution, as it is standard in the literature. Due to the nature of time series, the methodology employed consists in dividing each dataset into training, validation, and test sets to monitor training and assess performance effectively. Specifically, the division is 12/4/4 months for ETT, 15/3/4 months for ECL, and 28/10/10 months for WTH  as in \citep{zhou2023informer}.

\textbf{Algorithms: } We have maintained the same comparative framework as presented in the Informer paper, utilising the five algorithms selected by the original authors, being a representative sample of algorithms involving statistical methods and recurrent networks. This selection comprises ARIMA~\citep{ariyo2014stock}, Prophet~\citep{taylor2018forecasting}, LSTMa~\citep{bahdanau2014neural}, LST-Net~\citep{lai2018modeling}, and DeepAR~\citep{salinas2020deepar}. To better represent the state-of-the-art in time series forecasting, we have also included Reformer~\citep{kitaev2020reformer}, Autoformer~\citep{wu2021autoformer} and LogSparse self-attention~\citep{li2019enhancing}, alongside the Informer model~\citep{zhou2023informer}. All these models have undergone hyperparameter optimisation to determine the most suitable configuration for each dataset.

\textbf{Hardware and Code: } The experiments were conducted on an NVIDIA DGX-A100 workstation, which includes 1TB of RAM, two AMD Rome 7742 processors, and eight NVIDIA A100 graphics cards, each equipped with 40GB of memory. For implementation and experimentation, we developed a Python-based codebase, utilising the PyTorch framework for the development and training of neural networks. To ensure transparency and reproducibility, we have made the source code available through a dedicated GitHub repository within the ARI-DaSCI organisation\footnote{Link to our GitHub repository: \url{https://github.com/ari-dasci/S-LAM}}.

\subsection{\proposalextended\ compared to the state-of-the-art} \label{sec:informer-experimental-setup:results}

First, we will analyse a comparison between Informer and the number of parameters of the vanilla transformer incorporating  \proposal\ as attention mechanism, in order to test the improvement that arises from incorporating \proposal\ on vanilla architecture. It is important to highlight that the architecture accompanying our attention mechanism is exactly the vanilla architecture, based on Vaswani's~\citep{vaswani2017attention} original paper, without any modifications to the encoding-decoding framework or the positional encoding. The only difference between both models is the inclusion of \proposal\ instead of the full attention mechanism. In the experimentation presented in this section the number of encoding and decoding layers in the vanilla transformer is 3, as opposed to 4 encoding layers and 2 decoding layers in the Informer architecture~\citep{zhou2023informer}. 

\begin{table}[!hbt]
\centering
\resizebox{\textwidth}{!}{\begin{tabular}{ccccc}
\toprule
\textbf{Model}  & \textbf{Seq\_len} & \textbf{Parameters} & \multicolumn{1}{l}{\textbf{Ratio}} & \textbf{Num. encoding/decoding layers} \\
\midrule
Informer        & \multirow{2}{*}{48} & 12,455,745 & \multirow{2}{*}{2.1} & 4/2\\
Local Attention &  & 5,900,000 & & 6\\
\hline
Informer        & \multirow{2}{*}{168} & 12,455,745 & \multirow{2}{*}{2.1} & 3/3\\
Local Attention &  & 5,900,000 & & 6\\
\hline
Informer        & \multirow{2}{*}{336} & 12,455,745 & \multirow{2}{*}{2} & 4/2 \\
Local Attention &  & 6,000,000 & & 6\\
\hline
Informer        & \multirow{2}{*}{720} & 12,455,745 & \multirow{2}{*}{1.9} & 3/3\\
Local Attention &  & 6,400,000 & & 6\\
\hline
Informer        & \multirow{2}{*}{960} & 12,455,745 & \multirow{2}{*}{1.8} & 4/2\\
Local Attention &  & 6,800,000 & & 6\\
\bottomrule
\end{tabular}}
\caption{Number of parameters for Informer and \proposal\ model over the ECL dataset. The ratio is computed as \proposal\ model parameters divided by Informer parameters.}
\label{tab:parameters}
\end{table}

Table~\ref{tab:parameters} presents the size of both models on a fixed dataset while varying the size of the input window. Our model exhibits sensitivity to this variation, whereas Informer does not. The reason for this is that Informer has a fixed architecture that accepts an input window of size at most 512 time-steps. For smaller input window sizes, the input is extended with zeroes until reaching the mentioned size. Note that all the parameters in the Informer model are still used in the model independently of the size of the input window. In the optimal scenario, we achieve a parameter reduction factor of $2.1$, and in the worst scenario, a factor of $1.9$. Therefore, the model incorporating \proposal\ is significantly more compact in terms of the number of parameters compared to Informer, and consequently, in terms of memory usage. This compactness could allow us to concatenate more layers of our attention mechanism, thereby more accurately capturing long-term temporal dependencies. However, to highlight the effect of our attention mechanism in prediction accuracy, we have chosen to have a similar number of encoding and decoding layers as Informer. 

\begin{table*}[!hbt]
\centering
\fontsize{9pt}{9pt}\selectfont
\resizebox{\textwidth}{!}{\begin{tabular}{c|c|cc|cc|cc|cc|cc|cc|cc}
\toprule
\multicolumn{2}{c}{Methods} & \multicolumn{2}{|c}{\proposal} & \multicolumn{2}{|c}{Informer} & \multicolumn{2}{|c}{LogTrans} & \multicolumn{2}{|c}{Reformer} & \multicolumn{2}{|c}{Autoformer} & \multicolumn{2}{|c}{LSTMa} & \multicolumn{2}{|c}{LSTnet} \\
\midrule
\multicolumn{2}{c|}{Metric} & MSE & MAE & MSE & MAE & MSE & MAE & MSE & MAE & MSE & MAE & MSE & MAE & MSE & MAE \\
\midrule
\multirow{5}{*}{\rotatebox{90}{ETTh$_1$}} & 24 & \textbf{0.471} & \textbf{0.448} & 0.577 & 0.549 & 0.686 & 0.604 & 0.991 & 0.754 & 0.672 & 0.548 & 0.650 & 0.624 & 1.293 & 0.901 \\
    & 48 & \textbf{0.560} & \textbf{0.545} & 0.685 & 0.625 & 0.766 & 0.757 & 1.313 & 0.906 & 0.802 & 0.606 & 0.702 & 0.675 & 1.456 & 0.960 \\
    & 168 & 1.011 & 0.807 & \textbf{0.931} & \textbf{0.752} & 1.002 & 0.846 & 1.824 & 1.138 & 1.003 & 0.684 & 1.212 & 0.867 & 1.997 & 1.214 \\
    & 336 & \textbf{0.923} & \textbf{0.756} & 1.128 & 0.873 & 1.362 & 0.952 & 2.117 & 1.280 & 1.183 & 0.787 & 1.424 & 0.994 & 2.655 & 1.369 \\
    & 720 & \textbf{0.993} & \textbf{0.732} & 1.215 & 0.896 & 1.397 & 1.291 & 2.415 & 1.520 & 1.254 & 0.853 & 1.960 & 1.322 & 2.143 & 1.380 \\
\midrule
\multirow{5}{*}{\rotatebox{90}{ETTh$_2$}} & 24 & \textbf{0.588} & \textbf{0.543} & 0.720 & 0.665 & 0.828 & 0.750 & 1.531 & 1.613 & 0.923 & 0.998 & 1.143 & 0.813 & 2.742 & 1.457 \\
    & 48 & \textbf{0.780} & \textbf{0.692} & 1.457 & 1.001 & 1.806 & 1.034 & 1.871 & 1.735 & 1.530 & 1.207 & 1.671 & 1.221 & 3.567 & 1.687 \\
    & 168 & \textbf{3.209} & \textbf{1.479} & 3.489 & 1.515 & 4.070 & 1.681 & 4.660 & 1.846 & 3.956 & 1.958 & 4.117 & 1.674 & 3.242 & 2.513 \\
    & 336 & 6.289 & 2.192 & \textbf{2.723} & \textbf{1.340} & 3.875 & 1.763 & 4.028 & 1.688 & 4.016 & 1.897 & 3.434 & 1.549 & 2.544 & 2.591 \\
    & 720 & \textbf{2.835} & \textbf{1.204} & 3.467 & 1.473 & 3.913 & 1.552 & 5.381 & 2.015 & 4.235 & 2.004 & 3.963 & 1.788 & 4.625 & 3.709 \\
\midrule
\multirow{5}{*}{\rotatebox{90}{ETTm$_1$}} & 24 & \textbf{0.264} & \textbf{0.301} & 0.323 & 0.369 & 0.419 & 0.412 & 0.724 & 0.607 & 0.867 & 0.594 & 0.621 & 0.629 & 1.968 & 1.170 \\
    & 48 & 0.714 & 0.561 & \textbf{0.494} & \textbf{0.503} & 0.507 & 0.583 & 1.098 & 0.777 & 1.093 & 0.692 & 1.392 & 0.939 & 1.999 & 1.215 \\
    & 96 & \textbf{0.487} & \textbf{0.511} & 0.678 & 0.614 & 0.768 & 0.792 & 1.433 & 0.945 & 0.746 & 0.583 & 1.339 & 0.913 & 2.762 & 1.542 \\
    & 288 & \textbf{0.522} & \textbf{0.522} & 1.056 & 0.786 & 1.462 & 1.320 & 1.820 & 1.094 & 1.007 & 0.697 & 1.740 & 1.124 & 1.257 & 2.076 \\
    & 672 & \textbf{0.774} & \textbf{0.658} & 1.192 & 0.926 & 1.669 & 1.461 & 2.187 & 1.232 & 1.119 & 0.771 & 2.736 & 1.555 & 1.917 & 2.941 \\
\midrule
\multirow{5}{*}{\rotatebox{90}{Weather}} & 24 & \textbf{0.307} & \textbf{0.349} & 0.335 & 0.381 & 0.435 & 0.477 & 0.655 & 0.583 & 0.564 & 0.502 & 0.546 & 0.570 & 0.615 & 0.545 \\
    & 48 & 0.409 & \textbf{0.450} & \textbf{0.395} & 0.459 & 0.426 & 0.495 & 0.729 & 0.666 & 0.588 & 0.544 & 0.829 & 0.677 & 0.660 & 0.589 \\
    & 168 & \textbf{0.527} & \textbf{0.537} & 0.608 & 0.567 & 0.727 & 0.671 & 1.318 & 0.855 & 0.716 & 0.637 & 1.038 & 0.835 & 0.748 & 0.647 \\
    & 336 & \textbf{0.582} & \textbf{0.576} & 0.702 & 0.620 & 0.754 & 0.670 & 1.930 & 1.167 & 0.771 & 0.678 & 1.657 & 1.059 & 0.782 & 0.683 \\
    & 720 & \textbf{0.762} & \textbf{0.670} & 0.831 & 0.731 & 0.885 & 0.773 & 2.726 & 1.575 & 0.848 & 0.761 & 1.536 & 1.109 & 0.851 & 0.757 \\
\midrule
\multirow{5}{*}{\rotatebox{90}{ECL}} & 48 & \textbf{0.294} & \textbf{0.379} & 0.344 & 0.393 & 0.355 & 0.418 & 1.404 & 0.999 & 0.352 & 0.459 & 0.486 & 0.572 & 0.369 & 0.445 \\
    & 168 & \textbf{0.282} & \textbf{0.375} & 0.368 & 0.424 & 0.368 & 0.432 & 1.515 & 1.069 & 0.397 & 0.492 & 0.574 & 0.602 & 0.394 & 0.476 \\
    & 336 & \textbf{0.287} & \textbf{0.375} & 0.381 & 0.431 & 0.373 & 0.439 & 1.601 & 1.104 & 0.456 & 0.547 & 0.886 & 0.795 & 0.419 & 0.477 \\
    & 720 & \textbf{0.368} & \textbf{0.391} & 0.406 & 0.443 & 0.409 & 0.454 & 2.009 & 1.170 & 0.46 & 0.529 & 1.676 & 1.095 & 0.556 & 0.565 \\
    & 960 & \textbf{0.417} & \textbf{0.497} & 0.460 & 0.548 & 0.477 & 0.589 & 2.141 & 1.387 & 0.601 & 0.610 & 1.591 & 1.128 & 0.605 & 0.599 \\
\midrule
\multicolumn{2}{c}{Count} & \multicolumn{2}{|c}{43} & \multicolumn{2}{|c}{7} & \multicolumn{2}{|c}{0} & \multicolumn{2}{|c}{0} & \multicolumn{2}{|c}{0} & \multicolumn{2}{|c}{0} & \multicolumn{2}{|c}{0} \\
\bottomrule

\end{tabular}}
\caption{Multivariate LSTF results on the datasets described in Section~\ref{sec:informer-experimental-setup:setup}.}
\label{tab:exp.multivarResults}
\end{table*}

After comparing the sizes of the models, we proceed to examine the results of multivariate predictions. A comprehensive comparison of all algorithms for the ETTh1, ETTh2, ETTm1, Weather, and ECL datasets is presented in Table~\ref{tab:exp.multivarResults}. The most notable aspect of the table is the superior performance of \proposal. Out of 50 combinations of results and metrics, \proposal\ ranks highest in 43, followed by Informer, which leads in 7. When compared to Reformer, Autoformer and LogTrans, our approach significantly outperforms their results. This superiority is even more pronounced in comparison to LSTMa and LST-Net, particularly with longer prediction horizons, where the performance gap is the largest. This can be attributed to the capability of our architecture to capture long-term dependencies. While LSTM-based algorithms perform well with short prediction horizons, \proposal\ demonstrates substantial improvement and superior performance with longer prediction horizons, supporting the hypothesis that transformers are better equipped for LSTF problems.

Focusing on a direct comparison between \proposal\ and Informer allows for a more detailed case-by-case analysis. Overall, \proposal\ outperforms Informer, particularly for large forecast horizons in the ETT datasets. For example, in the case of ETTh2 with a forecast window of $720$, \proposal\ achieves an MSE of $2.8$ compared to Informer's $3.4$. Although \proposal\ also demonstrates superior performance in the ECL and Weather datasets, the margin of superiority is not as pronounced as in the ETT datasets.

\begin{figure}[!hbt]
	\centering
	\begin{subfigure}{0.45\linewidth}
		\includegraphics[width=\linewidth]{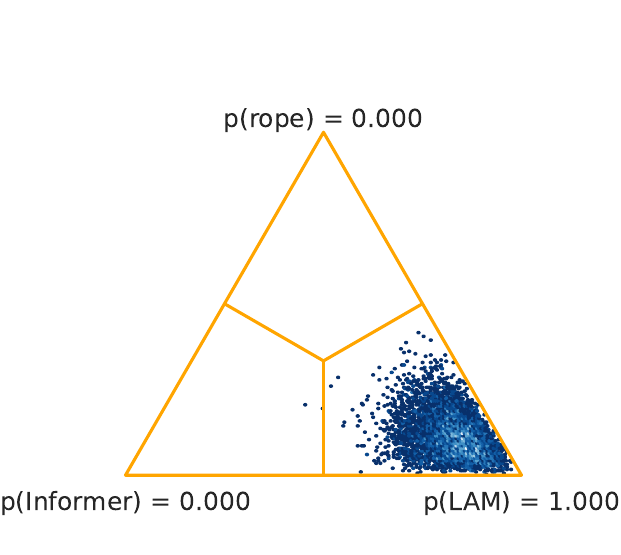}
		\caption{MSE Bayesian test.}
		\label{fig:mse-bayesian}
	\end{subfigure}
	\begin{subfigure}{0.45\linewidth}
		\includegraphics[width=\linewidth]{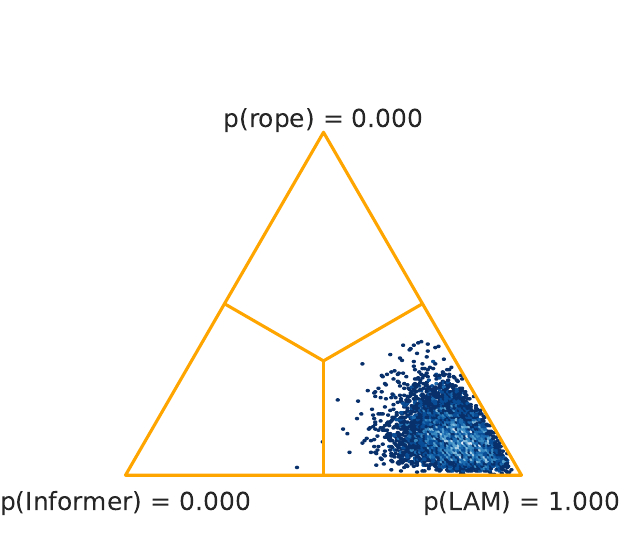}
		\caption{MAE Bayesian test.}
		\label{fig:mae-bayesian}
	\end{subfigure}
	\caption{Bayesian tests  for Informer vs \proposal.}
	\label{fig:bayesian-informer-leam}
\end{figure}

To better support this comparison, we conducted two Bayesian tests: one focusing on the MSE and the other on the MAE, using a $0.05$ ROPE. As illustrated in Figure~\ref{fig:mse-bayesian} and Figure~\ref{fig:mae-bayesian}, there is statistically significant superiority of \proposal\ with respect to the rest of algorithms considered. These significant differences underscore the robustness and effectiveness of our model compared to existing LSTF algorithms. Consequently, this provides strong validation for our experimentation methodology and supports the reliability of the conclusions drawn from our study.

\section{Comparing \proposal\ and ProbAttn within the same architecture framework} \label{sec:informer-experimental-setup:ProbAttention}

Following the previous comparison, it is pertinent to consider whether the two attention mechanisms, \proposal\ and Informer’s attention mechanism (ProbAttn) are competing on equal terms. In this Section we will conduct an experimental comparison between ProbAttention and \proposal\ within the transformer vanilla architecture. The same validation schema from the previous section is maintained. The results derived from this section will help us have an empirical comparison focusing exclusively on attention mechanisms.

\subsection{Experimental results}
\label{sec:informer:experimental-setup:probattention:results}

Informer is a significantly more complex model than the vanilla transformer integrating \proposal. The Informer model utilises its own positional encoding based on various time frequencies (seconds, minutes, hours, weeks, etc.) and incorporates layers such as one-dimensional convolutions to facilitate feature extraction from the dataset. Therefore, before proceeding with further experimentation, it is essential to compare \proposal\ and ProbAttn by aligning the frameworks under which both attention mechanisms are implemented. To this end, we will compare the results obtained by ProbAttn and \proposal\ within the vanilla transformer architecture, to isolate the effects of Informer's attention mechanism with respect to the characteristics of its architecture. This involves using only fully connected layers and the sine-based positional encoding from the original proposal by Vaswani et al~\citep{vaswani2017attention}, instead of the complex architecture of Informer.

\begin{table}[!hbt]
\centering
\resizebox{0.6\textwidth}{!}{\begin{tabular}{c|c|cc|cc}
\toprule
\multicolumn{2}{c}{Attention} & \multicolumn{2}{|c}{\proposal} & \multicolumn{2}{|c}{ProbAttn} \\
\midrule
\multicolumn{2}{c|}{Metric} & MSE & MAE & MSE & MAE \\
\midrule
\multirow{5}{*}{\rotatebox{90}{ETTh$_1$}} & 24 & \textbf{0.471} & \textbf{0.448} & 0.498 & 0.507 \\
                          & 48 & 0.560 & 0.545 & \textbf{0.501} & \textbf{0.515} \\
                          & 168  & \textbf{1.011} & 0.807 & 1.071 & \textbf{0.804} \\
                          & 336 & \textbf{0.923} & \textbf{0.756} & 1.002 & 0.796  \\
                          & 720 & \textbf{0.993} & \textbf{0.732} & 1.470 & 0.963  \\
\midrule
\multirow{5}{*}{\rotatebox{90}{ETTh$_2$}} & 24 & \textbf{0.588} & \textbf{0.543} & 0.742 & 0.683 \\
                          & 48 & \textbf{0.780} & \textbf{0.692} & 0.881 & 0.715 \\
                          & 168  & 3.209 & 1.479 & \textbf{2.652} & \textbf{1.380} \\
                          & 336 & 6.289 & 2.192 & \textbf{1.487} & \textbf{0.950} \\
                          & 720 & \textbf{2.835} & \textbf{1.204} & 3.234 & 1.314 \\
\midrule
\multirow{5}{*}{\rotatebox{90}{ETTm$_1$}} & 24 & \textbf{0.264} & \textbf{0.301} & 0.543 & 0.495 \\
                          & 48  & 0.714 & 0.561 & \textbf{0.698} & \textbf{0.560} \\
                          & 96  & \textbf{0.487} & \textbf{0.511} & 0.547 & 0.522 \\
                          & 288  & 0.522 & 0.522 & \textbf{0.502} & \textbf{0.505} \\
                          & 672 & 0.774 & \textbf{0.658} & \textbf{0.737} & 0.661 \\
\midrule
\multirow{5}{*}{\rotatebox{90}{ETTm$_2$}} & 24 & \textbf{0.453} & 0.507 & 0.482 & \textbf{0.502} \\
                          & 48  & 0.874 & 0.661 & \textbf{0.521} & \textbf{0.535} \\
                          & 96  & \textbf{0.520} & \textbf{0.526} & 0.788 & 0.641 \\
                          & 288  & \textbf{1.137} & \textbf{0.799} & 1.639 & 1.045 \\
                          & 672 & \textbf{1.340} & \textbf{0.933} & 10.35 & 2.609 \\
\midrule
\multirow{5}{*}{\rotatebox{90}{WTH}}  & 24 & \textbf{0.307} & \textbf{0.349} & 0.375 & 0.406 \\
                          & 48  & 0.409 & 0.450 & \textbf{0.406} & \textbf{0.449} \\
                          & 168  & 0.527 & 0.537 & \textbf{0.524} & \textbf{0.533} \\
                          & 336 & 0.582 & 0.576 & \textbf{0.563} & \textbf{0.563} \\
                          & 720 & \textbf{0.762} & \textbf{0.670} & 0.854 & 0.832 \\
\midrule
\multirow{5}{*}{\rotatebox{90}{ECL}} & 48  & 0.294 & 0.379 & \textbf{0.279} & \textbf{0.372} \\
                          & 168 & \textbf{0.282} & \textbf{0.375} & 0.288 & 0.382 \\
                          & 336 & \textbf{0.287} & \textbf{0.375} & 0.333 & 0.417 \\
                          & 720 & \textbf{0.368} & \textbf{0.391} & 0.406 & 0.443 \\
                          & 960 & \textbf{0.417} & \textbf{0.497} & 0.468 & 0.567 \\
\midrule
\multicolumn{2}{c}{Count}       & \multicolumn{2}{|c}{\textbf{37}}                 & \multicolumn{2}{|c}{23} \\
\multicolumn{2}{c|}{Accumulated} & \textbf{28.979} & \textbf{19.976} & 34.841 & 21.666 \\
\bottomrule

\end{tabular}}
\caption{MSE and MAE comparison between \proposal\ and ProbAttn mechanism. Best results in bold.}
\label{tab:informer-probatention}
\end{table}

The results presented in Table~\ref{tab:informer-probatention} provide a clear illustration of the performance differences. Notably, \proposal\ outperforms ProbAttn in 37 instances, while ProbAttn prevails in 23. However, it is crucial to also examine the cumulative error across all datasets and scenarios to establish a comprehensive metric. Specifically, the cumulative MSE for \proposal\ is $28.979$, compared to $34.841$ for ProbAttn. This indicates that, overall, \proposal\ exhibits a lower cumulative error than ProbAttn. Similarly, for MAE, \proposal\ achieves an accumulated value of $19.976$, whereas ProbAttn records $21.666$. Additionally, it is noteworthy that, on average, \proposal\ reduces the error of ProbAttn by approximately 25\% for longer prediction horizons. This finding corroborates the improvement and stability of our model's results under the LSTF problem.

\begin{figure}[!hbt]
	\centering
	\begin{subfigure}{0.45\linewidth}
		\includegraphics[width=\linewidth]{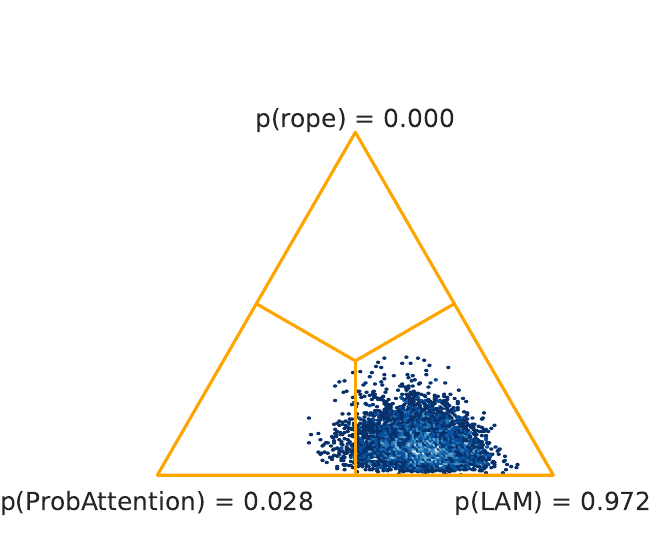}
		\caption{MSE Bayesian test.}
		\label{fig:mse-bayesian2}
	\end{subfigure}
	\begin{subfigure}{0.45\linewidth}
		\includegraphics[width=\linewidth]{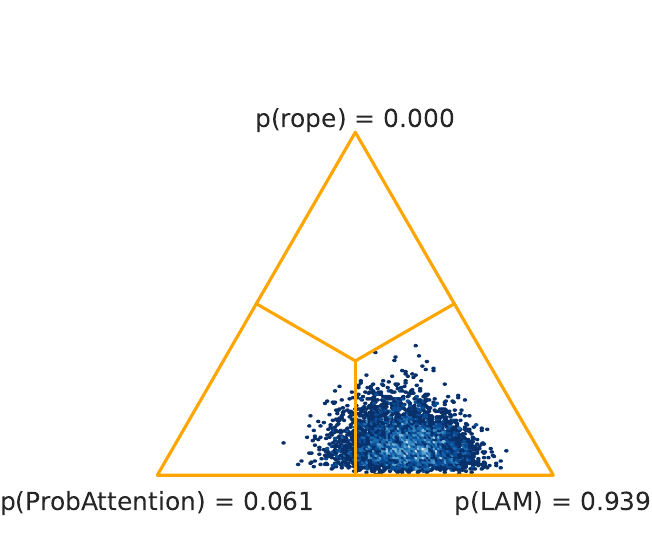}
		\caption{MAE Bayesian test.}
		\label{fig:mae-bayesian2}
	\end{subfigure}
	\caption{Bayesian tests ProbAttn vs \proposal.}
	\label{fig:bayesian-probattention-leam}
\end{figure}

To corroborate the data presented in the table, we include the corresponding Bayesian tests, as illustrated in Figure~\ref{fig:bayesian-probattention-leam}. These tests reveal that the differences are significant in favour of \proposal\, with a probability of 97.2\% for MSE and a probability of 93.9\% for MAE of \proposal\ being better than ProbAttn. These results support the effectiveness of \proposal\ within the same comparison framework.

In view of the results obtained in this experimentation, we argue that the complex architecture used in Informer~\citep{zhou2023informer} is not properly justified. 
Therefore, the question of determining a transformer architecture that successfully exploits the structure of long-term time series forecasting problems remains open. We note that in the paper that introduced Informer there is no comparison to the vanilla transformer and, thus, no empirical justification of the changes introduced at the architecture level -- the study focuses on comparing ProbAttn with FullAttention within their own architecture and comparing Informer to other models specifically designed for time series forecasting. For completeness, we also include comparisons between \proposal\ and FullAttention in Section~\ref{sec:lstf-experimental-setup}.

\subsection{On the suitability of benchmark models for the LSTF problem} \label{sec:informer-experimental-setup:conclusions}

As a conclusion to both experimentations, we have identified several areas for improvement in the literature. Firstly, all datasets utilised have a relatively small number of samples; none exceeds 100,000 instances (Table~\ref{tab:datasets-informer-description}). Transformer models for long-sequence time series forecasting, being highly complex architectures in terms of the number of weights, require a substantial amount of data to effectively train all parameters. The current datasets are insufficient in terms of sample size for this task, which produces results that are not sufficiently conclusive. Additionally, the datasets do not always encompass a large set of features. An increase in the number of variables or features inherently increases the complexity of multivariate prediction for LSTF, leading to more interesting problems. 

Analysing the ETT datasets in greater detail, we observe that there are four variants of the same dataset. Specifically, the nature of the data for ETTh1 and ETTm1 is the same as for ETTh2 and ETTm2; they represent the same problem but are derived from different sources. Additionally, each dataset has been sampled at 15-minute intervals (ETTm) and one-hour intervals (ETTh). In the context of long-term time series forecasting, it is crucial to provide as much granularity as possible to extend the prediction horizon effectively and meaningfully as otherwise the model will essentially be trained to predict averages. Thus, including two variants of the same problem with different sampling intervals does not appear to be justified.

Recall that the WTH dataset is deals with the prediction of weather-related values. Weather forecasting is known to be an exceedingly complex problem that requires extensive information to predict future instances with any degree of accuracy. Despite this complexity, weather forecasting is typically not aimed at long prediction horizons due to the inherent increase in error over time. Therefore, we consider this dataset to be unsuitable as a benchmark dataset for long-term time series forecasting. Additionally, the WTH dataset only provides hourly resolution, limiting the ability to extend the prediction horizon per instance without significantly increasing the risk of a high margin of error.

Lastly, let us analyse the ECL dataset, which represents the electricity consumption of 321 stations within the electricity grid over approximately two years with an hourly sampling frequency. While the scope of this dataset is relatively similar to that of the ETT dataset, certain distinctions can be observed. Notably, the dataset has a very low resolution of data collection. With only two years of data and an hourly sampling frequency to model a total of 321 variables representing all stations, the feature-to-instance ratio is excessively high. This imbalance makes it challenging to effectively address the problem using this dataset.

For these reasons, we believe that the currently used datasets, while useful for measuring model performance, are insufficient to establish a comprehensive benchmark for time series analysis in general, or for the LSTF problem in particular. Consequently, we recognise the necessity of obtaining new datasets that more accurately represent the LSTF problem to better evaluate model performance. This need will be addressed in the following section of this paper.

\section{Towards a suitable collection of datasets for LSTF benchmarking}
\label{sec:lstf-experimental-setup}

To evaluate the performance of \proposal\ in LSTF problems, we have highlighted a new set of experiments using datasets from various domains with real-world characteristics, allowing the exploration of inherent characteristics of transformers for LSTF. This section details the experimental framework, including the datasets and algorithms involved. Finally, we present and discuss the results obtained.

\subsection{Experimental Details}
\label{sec:lstf-experimental-setup:subsec:experimental-details}

\textbf{Datasets:} For this experiment, we selected four datasets: Individual Household Electric Power (IHEP)~\citep{misc_individual_household_electric_power_consumption_235}, New York Taxi (NYTaxi)~\citep{new_york_taxi}, Residential Power Battery (RPB)~\citep{residential_power_battery}, and Time-series Industrial Anomaly (TiNA)~\citep{tina_dasci_arcelor,mhbos}. The IHEP dataset contains over 200,000 samples with 9 features, representing household electricity consumption measurements over a 4-year period. The NYTaxi dataset includes over 2 million instances with 6 features, such as distance travelled, number of passengers, amount paid, and trip duration. These metrics are aggregated by summing all trips within the same hour, allowing the dataset to support predictions for subsequent hours, days, or weeks. The RPB dataset comprises more than 500,000 samples with 6 features, measuring the actual consumption of a population, solar power generation, and battery levels. Finally, the TiNA dataset includes over 38 million instances with 108 features, sampled every two seconds. This dataset is a real-world dataset provided by ArcelorMittal, reflecting data from an ore mining machine.

These datasets contain a large number of instances, allowing us to extend the prediction horizons substantially compared to the Informer-derived experimentation. Specifically, for the RPB and IHEP datasets, the window sizes are 180, 360, 720, 1440, and 4320, corresponding to 3 hours, 6 hours, 12 hours, 24 hours, and 3 days, respectively. For the NYTaxi dataset, the sequence sizes used were 120, 240, 720, 1440, and 2160, representing 2 days, 4 days, 12 days, 24 days, and 36 days, respectively. Finally, for the TiNA dataset, sequence sizes of 360, 720, 1440, 2880, and 5760 were used, corresponding to 3 hours, 6 hours, 12 hours, 24 hours, and 2 days, respectively. As shown, the predictions are long-term, aiming for the longest possible prediction horizon that is meaningful for each problem. The sequence size refers to both the input size of the network and the size of the prediction to be made.

\begin{table}[!hbt]
\centering
\begin{tabular}{llll}
\toprule
\textbf{Dataset} & \textbf{Instances} & \textbf{Features} & \textbf{Sampling}   \\
\midrule
IHEP   & 2,075,259     & 7        & 1 minute     \\
NYTaxi   & 2,091,896     & 6        & 1 hour     \\
RPB   & 614,880    & 4        & 1 minute \\
TiNA   & 38,000,000    & 103        & 1 second \\
\bottomrule
\end{tabular}
\caption{Datasets description. Number of instances, features and sampling rates are displayed in the table.}
\label{tab:datasets-uci-description}
\end{table}

\textbf{Multivariate prediction and error metrics:} 
Within this framework, the same target task and metrics outlined in the previous section are employed, alongside Bayesian tests. For validation, a data partitioning scheme of 70\% training data and 30\% testing data has been utilised.

\textbf{Algorithms:} In the experimentation provided in Section~\ref{sec:informer-experimental-setup}, Informer and LAM are the closest models in terms of performance. The other algorithms, including transformers such as LogTrans and Reformer, showed significantly poorer results. Consequently, we have decided to exclude these algorithms from this comparison to obtain concise and more interpretable results. Therefore, we have retained Informer and \proposal\, and added the vanilla transformer with its attention mechanism (FullAttention) to serve as a baseline.  For our proposed model and the vanilla attention mechanism, the optimal depth was determined through hyperparameter tuning. The transformer architecture employed is a fully connected, symmetric structure with an equal number of encoding and decoding layers. The depth was varied between 2, 3, 5, 7, and 10 layers, contingent upon memory capacity, to identify the best configuration. This parameter will be named as number of encoding-decoding layers for the rest of the section. By using the same architecture to baseline the vanilla attention mechanism, we ensure that the observed results can be solely attributed to our proposal and the network depth. The same standardisation and metrics introduced in the previous section are used for the present one. The train/validation/test split utilised in this study allocated 70\% of the data for training, with 15\% of this subset further designated for validation and the remaining 85\% for training. The remaining 30\% of the data was used as the test set.

\textbf{Hardware and Code:} The same hardware used in the experimentation of Section~\ref{sec:informer-experimental-setup} has been kept in order to obtain consistent results. The code is also hosted in the repository\footnote{https://github.com/ari-dasci/S-LAM} provided in previous sections.

\subsection{Results and Analysis}
\label{sec:lstf-experimental-setup:subsec:results-analysis}

In this section, we will analyse the results obtained from the datasets previously described in the framework explanation section. First, we refer to the results presented in Table~\ref{tab:lstf-multivariate-optimization}. A notable observation from this table is the presence of empty cells, which indicate execution problems due to memory limitations for certain parameter combinations. As evident, our proposed model is the only one capable of running in all scenarios. This is attributable to the low number of parameters in our model (shown in Section~\ref{sec:informer-experimental-setup}), resulting from its simplicity in the encoder-decoder architecture as opposed to Informer's architecture, and the $\Theta(n \log n)$ memory requirement for \proposal\ as opposed to the $\Theta(n^2)$ memory ussage of FullAttention. Out of the total 20 execution scenarios, Informer successfully executes in 17, while the vanilla transformer executes in 12. These findings highlight the superior performance of our model when dealing with large prediction horizons and highlights the importance of efficient attention mechanisms, both from a time and memory point of view.

In terms of performance, \proposal\ shows clear superiority. The table presents error results using MSE and MAE metrics. Out of the 40 cells that summarise both metrics, \proposal\ achieves the lowest error in 32 instances and ties in 2. Informer prevails in 4 instances, while the Vanilla transformer ties in 2 and wins in 2. When analysed by dataset, NYTaxi appears to be the easiest to address, as all three algorithms exhibit lower error rates for this dataset. Conversely, for the RPB dataset, Informer and our proposal yield similar results, indicating close performance. In the IHEP dataset, and particularly in the TiNA dataset, there is a significant performance gap favouring our proposal. Comparing with the Vanilla transformer, it is evident that while it runs in fewer instances, it generally performs better than Informer, except perhaps for the RPB dataset, where their results are comparable.

\begin{table}[!hbt]
\centering
\resizebox{!}{0.4\textwidth}{\begin{tabular}{c|c|ccc|cc|ccc}
\toprule
\multicolumn{2}{c}{Attention} & \multicolumn{3}{|c}{\proposal} & \multicolumn{2}{|c}{Informer} & \multicolumn{3}{|c}{Vanilla} \\
\midrule
\multicolumn{2}{c|}{Metric} & MSE & MAE & Layers & MSE & MAE & MSE & MAE & Layers \\
\midrule
\multirow{5}{*}{\rotatebox{90}{IHEP}} & 180 & \textbf{0.956} & \textbf{0.665} & 5 & 1.272 & 0.818 & 1.021 & 0.702 & 5\\
                                        & 360 & \textbf{0.942} & \textbf{0.670} & 7 & 1.295 & 0.828 & 0.979 & 0.706 & 5 \\
                                        & 720 & \textbf{0.973} & \textbf{0.678} & 5 & 1.244 & 0.817 & \textbf{0.973} & \textbf{0.678} & 3 \\
                                        & 1440 & \textbf{0.898} & \textbf{0.644} & 10 & - & - & - & - & - \\
                                        & 4320 & \textbf{0.977} & \textbf{0.698} & 3 & - & - & - & - & - \\
\midrule
\multirow{5}{*}{\rotatebox{90}{NYTaxi}} & 120 & \textbf{0.076} & \textbf{0.209} & 5 & 0.366 & 0.480 & 0.112 & 0.248 & 3 \\
                                        & 240 & 0.082 & 0.221 & 5 & 0.376 & 0.515 & \textbf{0.068} & \textbf{0.203} & 2 \\
                                        & 720 & \textbf{0.082} & \textbf{0.212} & 5 & 0.359 & 0.496 & 0.092 & 0.223 & 3\\
                                        & 1440 & \textbf{0.112} & \textbf{0.239} & 3 & 0.291 & 0.435 & - & - & - \\
                                        & 2160 & \textbf{0.103} & \textbf{0.223} & 10 & 0.4 & 0.505 & - & - & - \\
\midrule
\multirow{5}{*}{\rotatebox{90}{RPB}} & 180 & \textbf{2.999} & \textbf{1.321} & 3 & 3.402 & 1.431 & 3.123 & 1.335 & 2 \\
                                        & 360 & 2.985 & \textbf{1.304} & 7 & \textbf{2.90} & 1.344 & 3.015 & 1.327 & 3 \\
                                        & 720 & \textbf{2.996} & \textbf{1.319} & 3 & 3.383 & 1.419 & 3.054 & 1.334 & 2 \\
                                        & 1440 & 3.081 & \textbf{1.346} & 3 & \textbf{2.926}& 1.35 & 3.114 & 1.362 & 2 \\
                                        & 4320 & 3.027 & 1.325 & 2 & \textbf{2.822} & \textbf{1.314} & - & - & - \\
\midrule
\multirow{5}{*}{\rotatebox{90}{TiNA}}  & 360 & \textbf{1.843} & \textbf{0.858} & 3 & 2.044 & 0.954 & 1.854 & 0.863 & 3 \\
                                        & 720 & \textbf{1.849} & \textbf{0.863} & 7 & 2.148 & 0.976 & - & - & - \\
                                        & 1440 & \textbf{1.857} & \textbf{0.867} & 10 & 2.033 & 0.940 & - & - & - \\
                                        & 2880 & \textbf{1.827} & \textbf{0.850} & 10 & 2.020 & 0.944 & - & - & - \\
                                        & 5760 & \textbf{1.838} & \textbf{0.853} & 7 & - & - & - & - & - \\
\midrule
\multicolumn{2}{c}{Count}  & \multicolumn{3}{|c}{34} & \multicolumn{2}{|c}{4} & \multicolumn{3}{|c}{4}       \\
\multicolumn{2}{c}{Mean Accumulated}  & \textbf{1.475} & \textbf{0.768} &  & 1.722 & 0.915 & 1.582 & 0.816 &  \\
\bottomrule

\end{tabular}}

\caption{Multivariate forecasting results, MSE and MAE are presented in the table, the lower the better, best results in bold. Layers represent the best configuration of encoding-decoding layers used to obtain the presented results.}
\label{tab:lstf-multivariate-optimization}
\end{table}

Although the results presented in the aforementioned table demonstrate the validity and quality of our proposal, the influence of the number of layers within this comparison warrants further examination. The results may be highly sensitive to this configuration, potentially presenting a distorted image of the outcomes. As observed in Table~\ref{tab:lstf-multivariate-optimization}, there is a variety of optimal configurations for each dataset and sequence size. Therefore, we present the results with a fixed architecture for both Vanilla and \proposal in Table~\ref{tab:lstf-multivariate}. These results were obtained using the largest architecture that can run completely on all datasets for both models: 5 encoding-decoding layers for Vanilla and 7 encoding-decoding layers for \proposal. This approach provides a consistent performance comparison with a fixed architecture, allowing us to contrast these findings with the results in the aforementioned table.

As observed in Table~\ref{tab:lstf-multivariate}, \proposal~consistently emerges as the superior algorithm in this comparison. Specifically, it outperforms the other algorithms 31 times, whereas Informer obtains better results 5 times, and Vanilla wins 4 times. This demonstrates the clear superiority of our proposal over both the baseline and Informer models. Regarding the mean of the MSE and MAE across rows with valid values, Vanilla appears to have a slight edge. However, it is important to note that these results are significantly influenced by Vanilla's inability to run under numerous circumstances due to its computational demands. Consequently, these findings strongly establish the validity of our proposal, even when utilising a fixed architecture.

\begin{table}[!hbt]
\centering
\resizebox{!}{0.4\textwidth}{\begin{tabular}{c|c|cc|cc|cc}
\toprule
\multicolumn{2}{c}{Attention} & \multicolumn{2}{|c}{\proposal} & \multicolumn{2}{|c}{Informer} & \multicolumn{2}{|c}{Vanilla} \\
\midrule
\multicolumn{2}{c|}{Metric} & MSE & MAE & MSE & MAE & MSE & MAE \\
\midrule
\multirow{5}{*}{\rotatebox{90}{IHEP}} & 180 & \textbf{0.994} & \textbf{0.692} & 1.272 & 0.818 & 1.036 & 0.706 \\
                                        & 360 & \textbf{0.942} & \textbf{0.670} & 1.295 & 0.828 & 0.979 & 0.706 \\
                                        & 720 & \textbf{0.984} & \textbf{0.695} & 1.244 & 0.817 & 1.070 & 0.730 \\
                                        & 1440 & \textbf{1.030} & \textbf{0.739} & - & - & - & - \\
                                        & 4320 & \textbf{1.013} & \textbf{0.702} & - & - & - & - \\
\midrule
\multirow{5}{*}{\rotatebox{90}{NYTaxi}} & 120 & \textbf{0.147} & \textbf{0.285} & 0.366 & 0.480 & 0.162 & 0.286 \\
                                        & 240 & 0.119 & 0.263 & 0.376 & 0.515 & \textbf{0.105} & \textbf{0.248} \\
                                        & 720 & \textbf{0.112} & \textbf{0.244} & 0.359 & 0.496 & 0.119 & 0.248 \\
                                        & 1440 & \textbf{0.151} & \textbf{0.270} & 0.291 & 0.435 & - & - \\
                                        & 2160 & \textbf{0.182} & \textbf{0.312} & 0.4 & 0.505 & - & - \\
\midrule
\multirow{5}{*}{\rotatebox{90}{RPB}} & 180 & \textbf{2.977} & \textbf{1.312} & 3.402 & 1.431 & 3.196 & 1.387 \\
                                        & 360 & 2.985 & \textbf{1.304} & \textbf{2.90} & 1.344 & 3.024 & 1.325 \\
                                        & 720 & 3.168 & 1.378 & 3.383 & 1.419 & \textbf{3.153} & \textbf{1.372} \\
                                        & 1440 & 3.137 & 1.368 & \textbf{2.926} & \textbf{1.35} & - & - \\
                                        & 4320 & 3.093 & 1.342 & \textbf{2.822} & \textbf{1.314} & - & - \\
\midrule
\multirow{5}{*}{\rotatebox{90}{TiNA}}  & 360 & \textbf{1.862} & \textbf{0.870} & 2.044 & 0.954 & - & - \\
                                        & 720 & \textbf{1.849} & \textbf{0.863} & 2.148 & 0.976 & - & - \\
                                        & 1440 & \textbf{1.863} & \textbf{0.867} & 2.033 & 0.940 & - & - \\
                                        & 2880 & \textbf{1.851} & \textbf{0.860} & 2.020 & 0.944 & - & - \\
                                        & 5760 & \textbf{1.838} & \textbf{0.853} & - & - & - & - \\
\midrule
\multicolumn{2}{c}{Count}  & \multicolumn{2}{|c}{\textbf{31}} & \multicolumn{2}{|c}{5} & \multicolumn{2}{|c}{4}       \\
\multicolumn{2}{c}{Mean Accumulated}  & 1.514 & 0.794 & 1.722 & 0.915 & \textbf{1.427} & \textbf{0.778} \\
\bottomrule

\end{tabular}}

\caption{Multivariate forecasting results, MSE and MAE are presented in the table, the lower the better, best results in bold. \textbf{Fixed architectures} for \proposal \ and Vanilla have been employed, being 7 encoding-decoding layers for \proposal \ and 5 for Vanilla.}
\label{tab:lstf-multivariate}
\end{table}

To obtain an overall assessment, we utilised the mean cumulative error in both MAE and MSE as summary metrics. This average was calculated by considering only the rows with valid values for each algorithm, excluding those without results. This approach may inadvertently favour the Vanilla transformer and Informer, as they do not execute in scenarios with higher error rates (i.e., long prediction horizons), and therefore, lack results in those contexts. Despite this potential bias, our model still achieves the lowest average cumulative error in both MSE and MAE. The Vanilla transformer follows, with Informer exhibiting the highest average cumulative error.

\begin{figure}[H]
	\centering
	\begin{subfigure}{0.45\linewidth}
		\includegraphics[width=\linewidth]{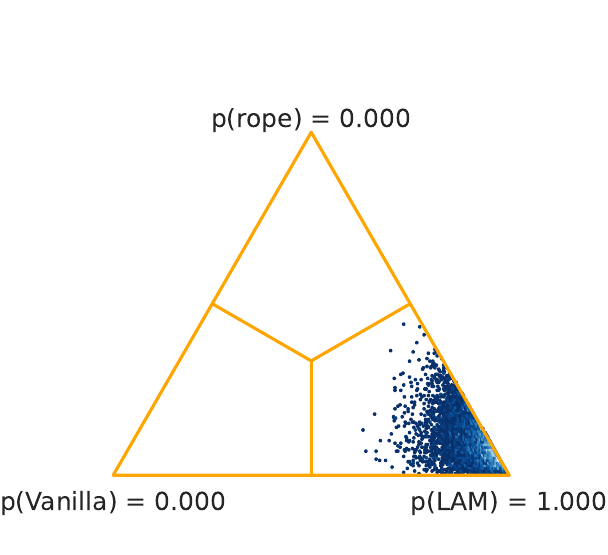}
		\caption{MSE Bayesian test.}
		\label{fig:mse-vanilla-lam-lstf}
	\end{subfigure}
	\begin{subfigure}{0.45\linewidth}
		\includegraphics[width=\linewidth]{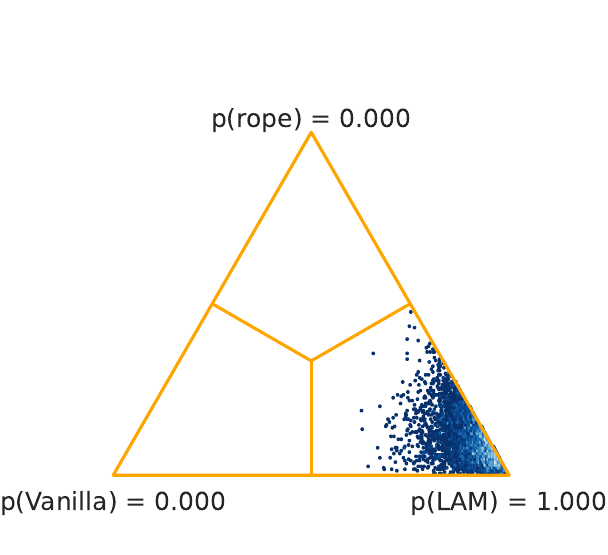}
		\caption{MAE Bayesian test.}
		\label{fig:mae-vanilla-lam-lstf}
	\end{subfigure}
	\caption{Bayesian tests for Vanilla vs \proposal.}
	\label{fig:vanilla-lam-lstf}
\end{figure}

\begin{figure}[H]
	\centering
	\begin{subfigure}{0.45\linewidth}
		\includegraphics[width=\linewidth]{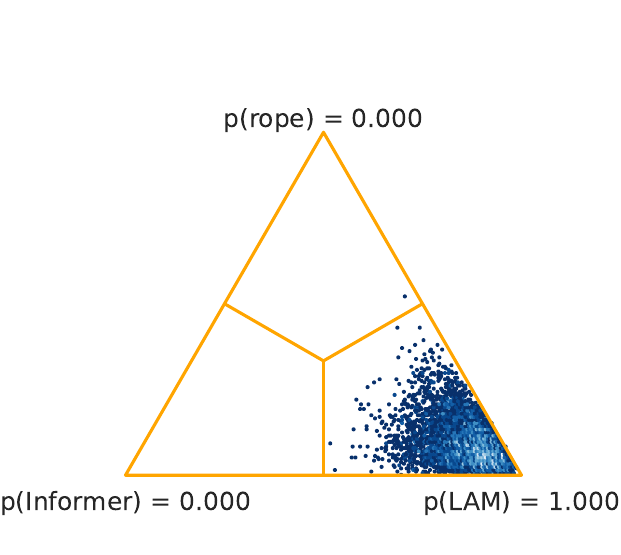}
		\caption{MSE Bayesian test.}
		\label{fig:mse-informer-lam-lstf}
	\end{subfigure}
	\begin{subfigure}{0.45\linewidth}
		\includegraphics[width=\linewidth]{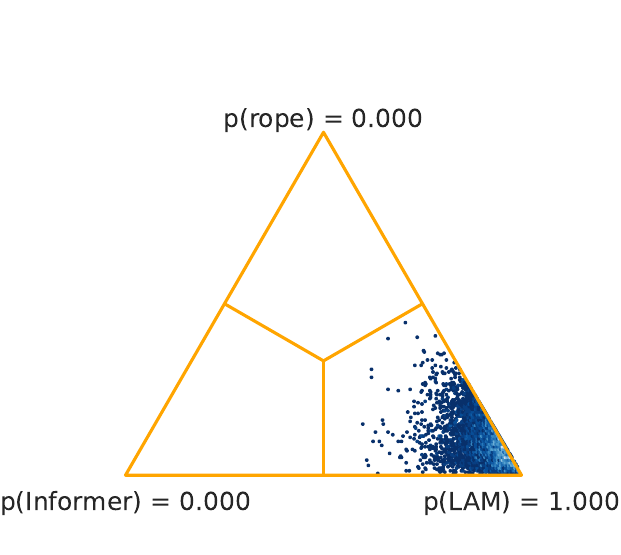}
		\caption{MAE Bayesian test.}
		\label{fig:mae-informer-lam-lsft}
	\end{subfigure}
	\caption{Bayesian tests for Informer vs \proposal.}
	\label{fig:informer-lam-lstf}
\end{figure}

To substantiate the argument that our algorithm outperforms both Informer and the Vanilla transformer, we conducted two Bayesian tests. The results, presented in Figure~\ref{fig:vanilla-lam-lstf} and Figure~\ref{fig:informer-lam-lstf}, clearly indicate statistically significant differences between the algorithms, favouring our proposal. These findings affirm the validity of our model.

In this experimentation, as mentioned at the beginning of this section, we varied the number of encoding-decoding layers to determine the optimal configuration. To study the effect of these architectural variations, the results are presented in bar charts, which can be viewed in Figure~\ref{fig:ihep-nytaxi-mse-layers} and Figure~\ref{fig:rpb-tina-mse-layers}, with a common legend in Figure~\ref{fig:legend-ihep-nytaxi-tina-rpb-layers}. It is important to note that, depending on the dataset size, some combinations of model layers may encounter memory issues and fail to execute, as observed in the TiNA dataset for sizes 1440, 2880, and 5760.

\begin{figure}[!hbt]
	\centering
	\begin{subfigure}{0.49\linewidth}
		\includegraphics[width=\linewidth]{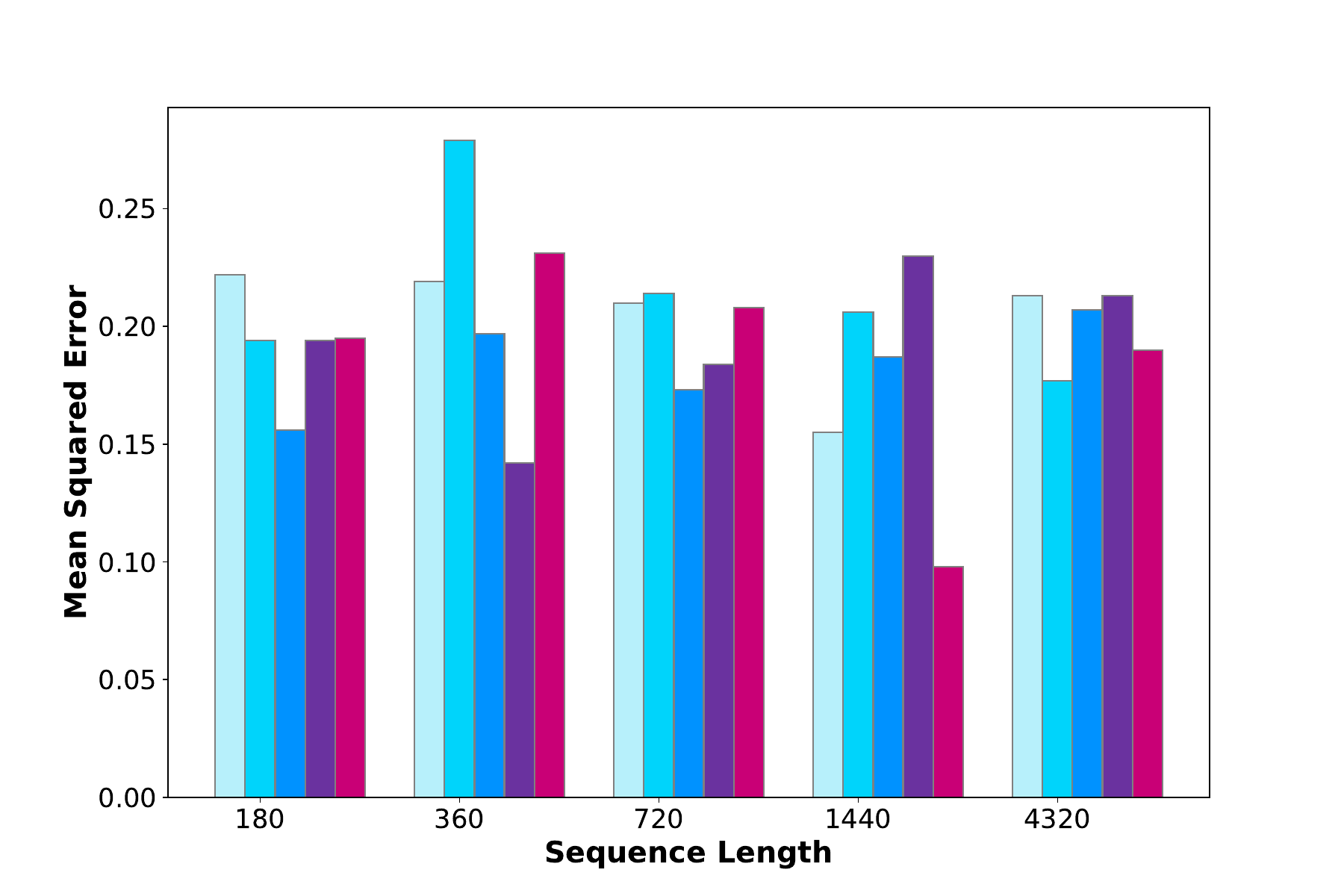}
		\caption{IHE dataset.}
		\label{fig:ihep-mse-layers}
	\end{subfigure}
	\begin{subfigure}{0.49\linewidth}
		\includegraphics[width=\linewidth]{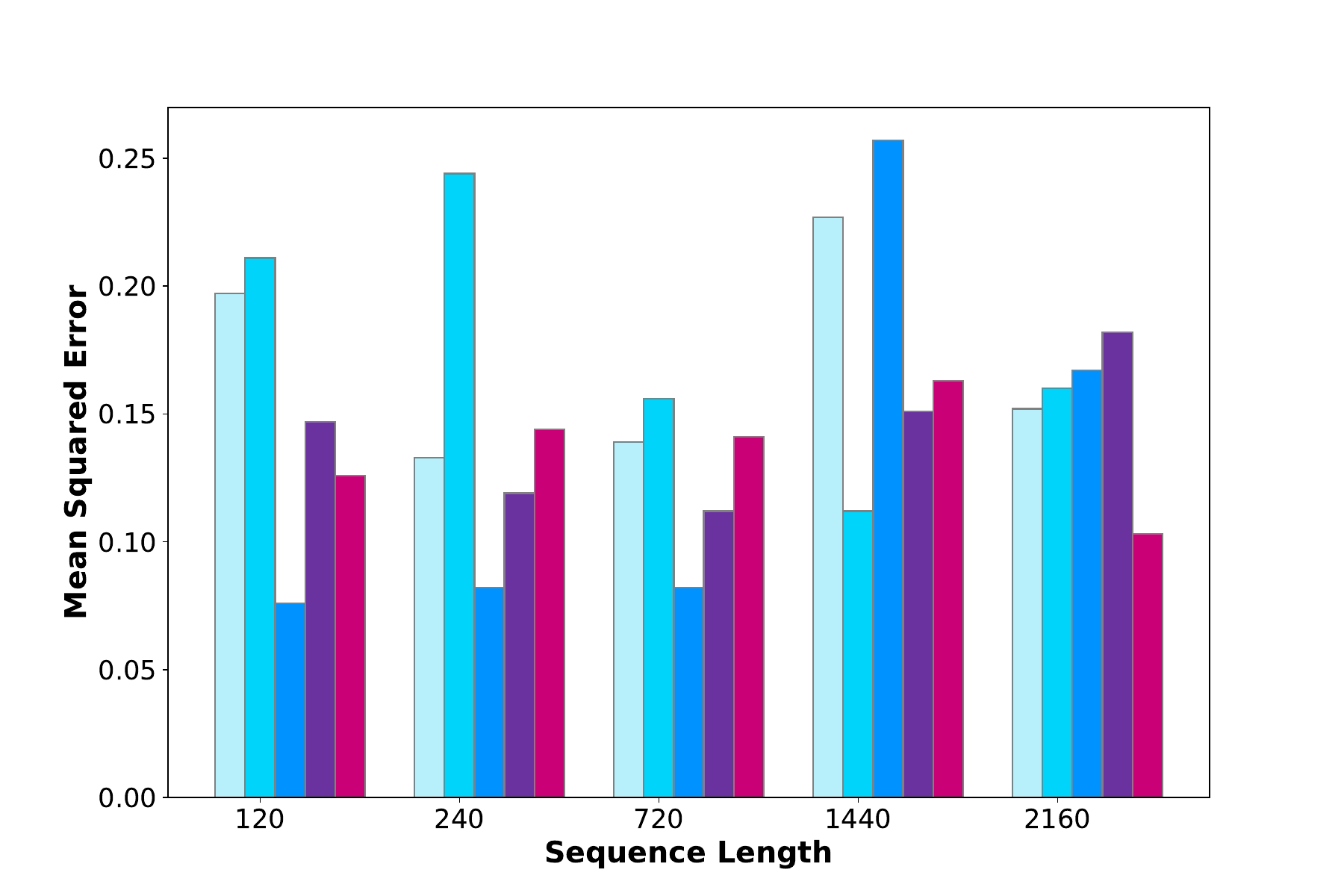}
		\caption{NYTaxi dataset.}
		\label{fig:nytaxi-mse-layers}
	\end{subfigure}
     \begin{subfigure}{\linewidth}
        \centering
        \includegraphics[scale=0.015]{figures/leyenda2.png}
        \caption{Legend for the barplots.}
        \label{fig:legend-ihep-nytaxi-tina-rpb-layers}
    \end{subfigure}
	\caption{Barplots for the MSE of \proposal\. MSE is presented in the barplot in the y-axis, sequence size (input and output) is presented in the x-axis.}
	\label{fig:ihep-nytaxi-mse-layers}
\end{figure}

The results exhibit some variation. Let us call by $N$-encoder-decoder the LAM model with $N$ encoding layers and $N$ decoding layers. For the IHEP dataset, the 2-encoder-decoder model is not optimal in any scenario. Specifically, the 5-encoder-decoder model performs best for size 180, the 7-encoder-decoder model for size 360, the 5-encoder-decoder model for size 720, the 10-encoder-decoder model for size 1440, and the 3-encoder-decoder model for size 4320. This indicates that there is no consistent trend across all sequence sizes for this dataset. However, the 10-encoder-decoder configuration yields favourable results for the longer sequence sizes of 1440 and 4320, suggesting that a higher number of layers may be beneficial for longer prediction horizons.

For the NYTaxi dataset, the most effective configurations generally consist of 2 and 3 encoding and decoding layers, with the 10-encoder-decoder architecture performing best for the longest prediction horizon. These results are expected, as the dataset does not represent a particularly complex problem, allowing it to be effectively addressed with a smaller architectural configuration in most cases.

For the RPB dataset, the optimal architecture is the 7-encoder-decoder model for sizes 180 and 360, the 3-encoder-decoder model for sizes 720 and 1440, and a tie between the 2-encoder-decoder and 4-encoder-decoder for size 4320. This dataset does not appear to represent a particularly complex problem over extended prediction horizons, as evidenced by the lack of improvement in results with an increased number of attention layers.

Finally, we present the results obtained from the TiNA dataset, which is the most complex due to its large number of instances and variables. Consequently, it is the dataset that stands to benefit the most from architectures with more parameters. For this dataset, the 10-encoder-decoder architecture is unable to execute for sizes 1440 and above. For size 360, the optimal architecture is the 3-encoder-decoder configuration, while the 7-encoder-decoder model performs best for all other cases. The TiNA dataset clearly demonstrates the performance advantages of higher architectural depth, particularly for the sequence size of 5760.

\begin{figure}[!hbt]
	\centering
	\begin{subfigure}{0.45\linewidth}
		\includegraphics[width=\linewidth]{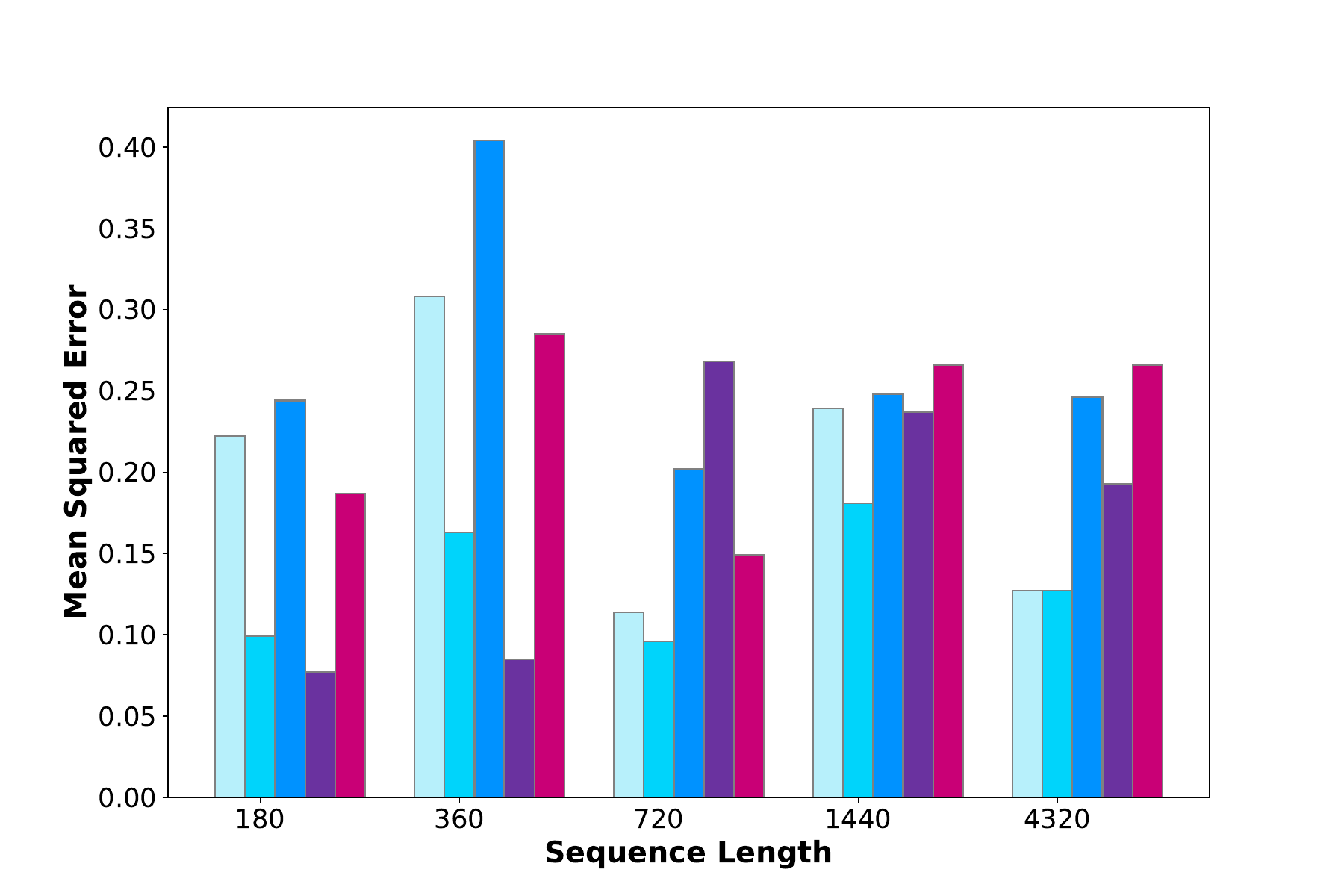}
		\caption{RPB dataset.}
		\label{fig:rpb-mse-layers}
	\end{subfigure}
	\begin{subfigure}{0.45\linewidth}
		\includegraphics[width=\linewidth]{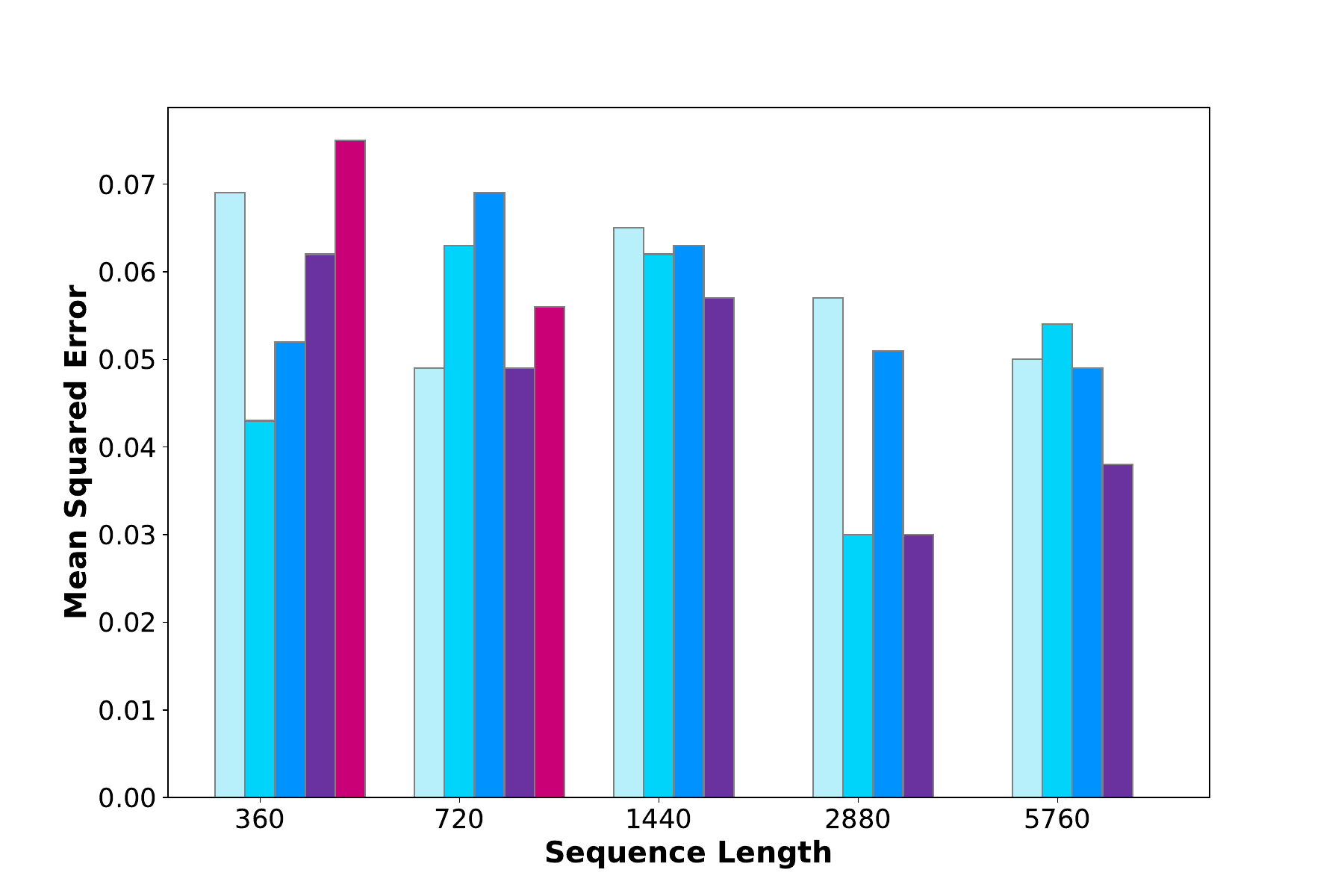}
		\caption{TiNA dataset.}
		\label{fig:tina-mse-layers}
	\end{subfigure}
    \begin{subfigure}{\linewidth}
        \centering
        \includegraphics[scale=0.015]{figures/leyenda2.png}
        \caption{Legend for the barplots.}
        \label{fig:legend-ihep-nytaxi-tina-rpb-layers2}
    \end{subfigure}
	\caption{Barplots for the MSE of \proposal\. MSE is presented in the barplot in the y-axis, sequence size (input and output) is presented in the x-axis.}
	\label{fig:rpb-tina-mse-layers}
\end{figure}

\section{Conclusions and future work}
\label{sec:conclusions-future-work}

This paper proposes a novel attention mechanism for transformers, specifically designed for long sequence time series forecasting. This mechanism leverages the locality of the data and maintains a theoretical efficiency of $\Theta(log(n))$. Our proposed mechanism has demonstrated superior performance compared to state-of-the-art attention-based mechanisms and models, including Informer~\citep{zhou2023informer}, LogTrans~\citep{li2019enhancing}, Reformer~\citep{kitaev2020reformer}, and Autoformer~\citep{wu2021autoformer}, as well as LSTM models. 

Additionally, we have highlighted the necessity for new and improved datasets for the evaluation of the LSTF problem, in contrast to existing benchmarks. To address this need, we have introduced new datasets that facilitate a more accurate assessment of model quality in these scenarios.

On larger and more challenging datasets, \proposal~has consistently outperformed not only Informer but also the baseline represented by the vanilla model across nearly all scenarios. Furthermore, due to the simplicity of the model accompanying our mechanism, our proposal is uniquely capable of executing effectively in all the studied scenarios. Additionally, we observed that the results exhibit low sensitivity to the depth of the architecture employed. By maintaining a fixed architecture size, we confirmed that the hypothesis remains valid.

Conducted experiments have also yielded a strong model performance and improved metrics when employing a larger number of attention layers for complex problems. This suggests that, for problems that are challenging to model or that involve long prediction horizons, the results benefit significantly from utilising a greater number of attention layers.

Finally, we identified several areas for improvement and further research to enhance the performance of future models and advance the field. Firstly, as an existing challenge, it is essential to establish a comprehensive benchmark that includes datasets, methodologies, and metrics to evaluate new proposals consistently. Maintaining an up-to-date table of state-of-the-art model results is also crucial. Additionally, our study exclusively proposes an attention mechanism without a dedicated architecture; thus, future work should focus on developing and integrating such architectures. Lastly, considering the promising results obtained with a basic model employing a novel attention mechanism, it is imperative to conduct rigorous ablation studies. These studies will help discern which components positively impact the results and which do not, thereby guiding the development of more effective models.

\subsubsection*{Acknowledgements}

This publication is part of the project ``Ethical, Responsible, and General Purpose Artificial Intelligence: Applications In Risk Scenarios'' (IAFER) Exp.:TSI-100927-2023-1 funded through the creation of university-industry research programs (Enia Programs), aimed at the research and development of artificial intelligence, for its dissemination and education within the framework of the Recovery, Transformation and Resilience Plan from the European Union Next Generation EU through the Ministry of Digital Transformation and the Civil Service. This work was also partially supported by Knowledge Generation Projects, funded by the Spanish Ministry of Science, Innovation, and Universities of Spain under the project PID2023-150070NB-I00. Ignacio Aguilera-Martos was supported by the Ministry of Science of Spain under the FPI programme PRE2021-100169.

\newpage
\appendix

\section{Theoretical proofs}
\label{appendix:theoretical-proofs}

\paragraph{\textbf{Proposition}} \label{prop:perm_appendix}
In the case of the full attention mechanism, for any positive integers $h$ and $d_a$ and any permutation matrix $P_\pi$, we have 
\begin{equation*}
    \operatorname{MultiHead}(P_\pi Q, P_\pi K, P_\pi V) =  P_\pi \operatorname{MultiHead}(Q, K, V).
\end{equation*}

\begin{proof}
We first prove the proposition for $h = 1$, that is, we show that
\begin{equation*}
    \operatorname{FullAttn}(P_\pi Q, P_\pi K, P_\pi V) =  P_\pi \operatorname{FullAttn}(Q, K, V)
\end{equation*}
 Let  $a_{i,j}$ and $\hat{a}_{i,j}$ be the attention scores for the original inputs and the permuted inputs, respectively. Note that $(P_\pi Q)(P_\pi K)^T = P_\pi Q   K^T P_\pi^T$. That is, $\hat{a}_{i,j} = a_{\pi(i),\pi(j)}$. Thus, the $i$-th row of $\operatorname{FullAttn}(P_\pi Q, P_\pi K, P_\pi V)$ is the vector $\sum_{j = 1}^n a_{\pi(i),\pi(j)} V_{\pi(j)} = \sum_{j = 1}^n a_{\pi(i),j} V_{j}$ as we wanted. Finally, we note that $\operatorname{Concat}(P_\pi head_1, \ldots, P_\pi head_h) W^O = P_\pi \operatorname{Concat}( head_1, \ldots, head_h) W^O$, as the tensors $head_1, head_2, \ldots, head_h$ are concatenated on the last dimension and $P_\pi$ permutes the second to last dimension.
\end{proof}

\paragraph{\textbf{Lemma}} \label{lem:TM_appendix}
 Let $r \in \{0,1,\ldots,s-1\}$. For each $i_1 \in \{0, 1, \ldots, L-1\}$ and  $j_1 \in \{0, 1, \ldots, 2L-2\}$, let $i = rL+i_1$ and $j = (r-1)L+1+j_1$. We have $i-L+1 \le j \le i$ if and only if $i_1 \le j_1 \le i_1+L-1$.

\begin{proof}
    We note that $j - i+L-1 = j_1 - i_1$ and $i-j = i_1+L-1-j_1$ by definition of $i$ and $j$. Thus, we have $j_1 - i_1 \ge 0$ if and only if $j-i + L-1 \ge 0$. Equivalently, $i_1 \le j_1$ if and only if $i-L+1 \le j$, which deals with the lower bounds of the statement. Analogously, $i-j \ge 0$ if and only if $i_1+L-1-j_1 \ge 0$ or, equivalently, $j \le i$ if and only if $j_1 \le i_1+L-1$, which finishes the proof.
\end{proof}

\paragraph{\textbf{Lemma}} \label{lem:TQ-TK_appendix}
    For each $r \in \{0, 1, \ldots, s-1\}$, the matrix defined as $T_A[r, \cdot, \cdot] = T_Q[r,\cdot,\cdot] T_K[r,\cdot,\cdot]^T$ has dimensions $L \times (2L-1)$. For each $i \in \{0,1,\ldots, n-1\}$ and $j$ with $i-L+1\le j \le i$, setting $i_1 = i \pmod L$, $r = (i -i_1)/L$ and $j_1 = j-(r-1)L-1$, we have $r \in \{0, 1, \ldots, s-1\}$, $i_1 \in \{0, 1, \ldots, L-1\}$ and $j_1 \in \{0,1,\ldots,2L-2\}$. Moreover, the equality $\langle Q_i, K_j \rangle = T_A[r, i_1, j_1]$ holds.

\begin{proof}
   Recall that $T_Q$ has dimension $s \times L \times d_q$ and $T_K$ has dimension $s \times (2L-1) \times d_q$, so $T_K[r,\cdot,\cdot]^T$ has dimensions $d_q \times (2L-1)$ and $T_Q[r,\cdot,\cdot] T_K[r,\cdot,\cdot]^T$ has dimensions $L \times (2L-1)$ as we wanted. For the second claim, first, we check that indeed $i_1 \in \{0,1,\ldots, L-1\}$ (which follows by the definition of $\mod L$) and $r \in \{0,1,\ldots,s-1\}$ (which follows from $s = n /L > (i-i_1)/L = r$). To prove that $j_1 \in \{0, 1, \ldots, 2L-2\}$, we use $i-L+1\le j \le i$, $i = i_1 + rL$ and $j_1 = j-(r-1)L-1$, obtaining
   \begin{equation*}
       0 \le j-i+L-1 = j-i_1-rL+L-1 = j_1 -i_1 \le j_1
   \end{equation*}
   and 
   \begin{equation*}
   \begin{aligned}
        j_1 & =  j-(r-1)L-1 \le i -(r-1)L-1 = i_1 + rL-(r-1)L-1    \\ & = i_1 +L-1 \le L-1+L-1 = 2L-2.       
   \end{aligned}
   \end{equation*}
   Finally, note that by definition of $T_Q$ and $T_K$, we have $Q_i = T_Q[r, i_1, \cdot]$ and $K_j = T_K[r, j_1, \cdot]$ as $(r-1)L+1+j_1 = j\ge 0$. Note that $T_A[r, i_1, j_1] = \langle T_Q[r, i_1, \cdot], T_K[r, j_1, \cdot] \rangle$ by definition of matrix multiplication. Thus, we conclude that $\langle Q_i, K_j \rangle = T_A[r, i_1, j_1]$.
\end{proof}

\paragraph{\textbf{Lemma}} \label{lem:T_S_appendix}
    The tensor defined as $T_S = \operatorname{softmax}((T_Q T_K^T + T_M) / \sqrt{d_q})$, where $\operatorname{softmax}$ is applied on the last dimension, has dimensions $s \times L \times (2L-1)$. Let $S$ be the matrix $\operatorname{softmax}\left( (QK^T + M)/\sqrt{d_q}\right)$. For each  $i \in \{0,1,\ldots, n-1\}$ and $j$ with $i-L+1\le j \le i$, setting $i_1 = i \pmod L$, $r = (i -i_1)/L$ and $j_1 = j-(r-1)L-1$, we have $S[i,j] = T_S[r, i_1, j_1]$. Moreover, all the other entries of $S$ and $T_A$ are zero.

\begin{proof}
    The dimensions of $T_S$ follow from the dimensions of $T_A = T_Q T_K^T$, see Lemma~\ref{lem:TQ-TK}. Let $i,j \in \{0,1,\ldots, n-1\}$. We consider the change of variables $i_1 = i \pmod L$, $r = (i -i_1)/L$ and $j_1 = j-(r-1)L-1$. For the rest of the proof we distinguish two cases:
    \begin{itemize}
        \item The case when $i-L+1\le j \le i$. In view of Lemma~\ref{lem:TQ-TK}, we have $r \in \{0, 1, \ldots, s-1\}$, $i_1 \in \{0, 1, \ldots, L-1\}$ and $j_1 \in \{0,1,\ldots,2L-2\}$, and, moreover, $\langle Q_i, K_j \rangle = T_A[r, i_1, j_1]$, where $T_A = T_Q T_K^T$.
        \item The case when the $i-L+1\le j \le i$ does not hold. In such case, $M[i,j] = -\infty$ by definition, so $(QK^T+M)[i,j] = -\infty$.  Moreover, either $j_1 \not \in \{0, 2, \ldots, 2L-2\}$ and $T_M[r,i_1,j_1]$ is not defined, or $j_1 \in \{0, 2, \ldots, 2L-2\}$. In the latter case, we are in the setting of Lemma~\ref{lem:TM}, so, as $i-L+1\le j \le i$ does not hold, the inequality $i_1 \le j_1 \le i_1+L-1$ does not hold. From the definition of $T_M$, we conclude that $(T_A + T_M)[r,i_1, j_1] = -\infty$. That is, we have shown that either $(T_A + T_M)[r,i_1, j_1]$ is not defined and $(QK^T +M)[i,j] = -\infty$, or $(T_A + T_M)[r,i_1, j_1] =-\infty$ and $(QK^T +M)[i,j] = -\infty$.
    \end{itemize}
    Applying the operator $\operatorname{softmax}$ to the $i$-th row of $(QK^T+M)$, leads to all entries equal to zero except for those entries indexed by $(i,j)$ with $i-L+1\le j \le i$ (a total of $L$ entries). Analogously, applying the operator $\operatorname{softmax}$ to $(T_A+T_M)[r, i_1, \cdot]$, leads to all entries equal to zero except for those entries indexed by $(i_1,j_1)$ with $i_1 \le j_1 \le i_1+L-1$ (a total of $L$ entries). Moreover, $L$ non-zero entries are the same in both cases. Summarising, we have shown that: 
    \begin{itemize}
        \item when $i-L+1\le j \le i$, we have $S[i,j] = T_S[r, i_1, j_1]$.
        \item when $i-L+1\le j \le i$ does not hold,  we have $S[i,j] = 0$ and either $T_S[r, i_1, j_1]$ is not defined (case $j_1 \not \in \{0,1,\ldots,2L-2\}$) or $T_S[r, i_1, j_1] = 0$. 
    \end{itemize}
    This dichotomy concludes the proof.
\end{proof}

\paragraph{\textbf{Lemma}} \label{lem:T_V_appendix}
    The tensor defined as $T_{LAM} = \operatorname{softmax}((T_Q T_K^T + T_M) / \sqrt{d_q})T_V$, where $\operatorname{softmax}$ is applied on the last dimension, has dimensions $s \times L \times d_{model}$. For each $i \in \{0,1,\ldots, n-1\}$, setting $i_1 = i \pmod L$ and $r = (i -i_1)/L$, we have $\proposal(Q,K,V)[i, \cdot] = T_{LAM}[r, i_1, \cdot]$.

\begin{proof}
    Recall that $T_S = \operatorname{softmax}((T_Q T_K^T + T_M) / \sqrt{d_q})$ has dimensions $s \times L \times (2L-1)$, see Lemma~\ref{lem:T_S}, and $T_V$ has dimensions $s \times (2L-1) \times d_{model}$. Thus, $T_{LAM}$ has dimensions $s \times L \times d_{model}$. Now, le t$i \in \{0,1,\ldots, n-1\}$ and  set $i_1 = i \pmod L$ and $r = (i -i_1)/L$. We note that $i_1 \in \{0,1,\ldots, L-1\}$ and $r \in \{0,1,\ldots, s-1\}$.  By Lemma~\ref{lem:T_S}, for all $j_1 \in \{0,1,\ldots,2L-2\}$, we have $T_S[r, i_1, j_1] = S[i, j]$ for $j = r(L-1)+1+j_1$, and $S[i,j] = 0$ for other values of $j$. Therefore, 
    we find that 
    \begin{equation*}
        T_{LAM}[r, i, \cdot] =  T_S[r, i_1, \cdot] T_V[r, \cdot, \cdot] = S[i, \cdot] V,
    \end{equation*}
    which equals $\proposal(Q,K,V)[i, \cdot]$, as we wanted to prove.
\end{proof}

\paragraph{\textbf{Theorem}} \label{thm:main_appendix}
    Algorithm 1 computes $\operatorname{\proposal}(Q, K, V)$ in time and memory $\Theta(n L)$, where time complexity is measured in the number of dot products of vectors performed.

\begin{proof}
    The fact that Algorithm~\ref{alg:local} does indeed compute the attention scores that we are interested in and that the output is that of $\proposal(Q,K,V)$ follows from Lemma~\ref{lem:T_V}. We have performed $s L \cdot (2L-1) =  (2L-1) n = \Theta(n \log n)$ dot products. This computation takes time and memory $\Theta(s L^2) = \Theta(L n)$. 
\end{proof}

\paragraph{\textbf{Corollary}}
    For $L = \lceil 4\log n \rceil$, Algorithm 1 computes $\operatorname{\proposal}(Q, K, V)$ in time and memory $\Theta(n \log n)$.

\begin{proof}
    This follows from instantiating $L$ in Theorem~\ref{thm:main}.
\end{proof}

\bibliographystyle{unsrt}
\bibliography{references}

\end{document}